\documentclass[conference]{IEEEtran}  

\IEEEoverridecommandlockouts                              

\usepackage[numbers]{natbib}
\usepackage{amsthm}
\usepackage{times}
\usepackage{multicol}
\usepackage[bookmarks=true]{hyperref}
\usepackage{xcolor}
\usepackage{hyperref}
\usepackage{amsmath, amssymb}
\usepackage{amsfonts}
\usepackage{graphicx}
\usepackage{siunitx}
\usepackage{standalone}
\usepackage{booktabs}
\usepackage[ruled,vlined,linesnumbered,noend]{algorithm2e}
\usepackage{mdframed}
\usepackage{fancyvrb,multirow}
\usepackage{soul}
\usepackage{dsfont,mathabx}
\usepackage{array, booktabs}
\usepackage{subcaption}
\usepackage{booktabs}
\usepackage{makecell}
\usepackage{tikz}
\usepackage{pgfplots}
\usepackage[labelfont=bf,format=plain]{caption}
\usepackage[para,online,flushleft]{threeparttable}

\newtheorem{theorem}{Theorem}

\newtheorem{lemma}{Lemma}
\theoremstyle{definition}

\newtheorem{definition}{Definition}
\newtheorem{remark}{Remark}

\newtheorem{problem}{Problem}

\title{iHERO: Interactive Human-oriented Exploration and Supervision Under Scarce Communication}

  \author{Zhuoli Tian, Yuyang Zhang, Jinsheng Wei and Meng Guo$^*$
  \\College of Engineering, Peking University
\thanks{$^*$Corresponding author: Meng Guo, {\tt\small meng.guo@pku.edu.cn}.}
}

\begin{document}
\maketitle
\thispagestyle{empty}
\pagestyle{empty}


\begin{abstract}
  Exploration of unknown scenes before human entry is essential
  for safety and efficiency in numerous scenarios,
  e.g., subterranean exploration, reconnaissance, search and rescue missions.
  Fleets of autonomous robots are particularly suitable for this task,
  via concurrent exploration, multi-sensory perception and autonomous navigation.
  Communication however among the robots can be severely restricted
  to only close-range exchange via ad-hoc networks.
  Although some recent works have addressed the problem of
  collaborative exploration under restricted communication,
  the crucial role of the human operator has been mostly neglected.
  Indeed, the operator may:
  (i) require timely update  regarding the exploration progress and fleet status;
  (ii) prioritize certain regions; and
  (iii) dynamically move within the explored area;
  To facilitate these requests, this work proposes an interactive
  human-oriented online coordination framework
  for collaborative exploration and supervision under scarce communication (iHERO).
  The robots switch smoothly and optimally among
  fast exploration, intermittent exchange of map and sensory data,
  and return to the operator for status update.
  It is ensured that these requests are fulfilled online
  interactively with a pre-specified latency.
  Extensive large-scale human-in-the-loop simulations and hardware experiments are performed
  over numerous challenging scenes,
  which signify its performance such as explored area and efficiency,
  and validate its potential applicability to real-world scenarios. 
  The videos are available on \url{https://zl-tian.github.io/iHERO/}.
\end{abstract}

\section{Introduction}\label{sec:intro}
Exploration of unknown and hazardous scenes before allowing humans inside
is crucial for their safety, which can also improve
the performance of subsequent tasks with an accurate model.
Fleets of UAVs and UGVs have been deployed as {extended eyes and arms},
e.g., to explore planetary caves in~\citet{petravcek2021large,klaesson2020planning};
search and rescue after earthquakes in~\citet{couceiro2017overview}.
Many collaborative exploration strategies have been
proposed along with the advances in autonomous navigation and perception,
e.g.,~\citet{yamauchi1997frontier,hussein2014multi,colares2016next,
  zhou2023racer,patil2023graph,guo2015multi}.
However, they often assume an all-to-all communication
among the robots,
where any map observed locally by one robot
is immediately available to other robots.
This could be impractical in many aforementioned scenes
where the communication facilities are unavailable or severely degraded.
In such cases, the robots can only exchange information via ad-hoc networks in close proximity.
This imposes great challenges on the fleet coordination as communication events and exploration tasks
are now closely dependent,
thus should be planned simultaneously.

\begin{figure}[t]
  \centering
  \includegraphics[width=0.90\linewidth]{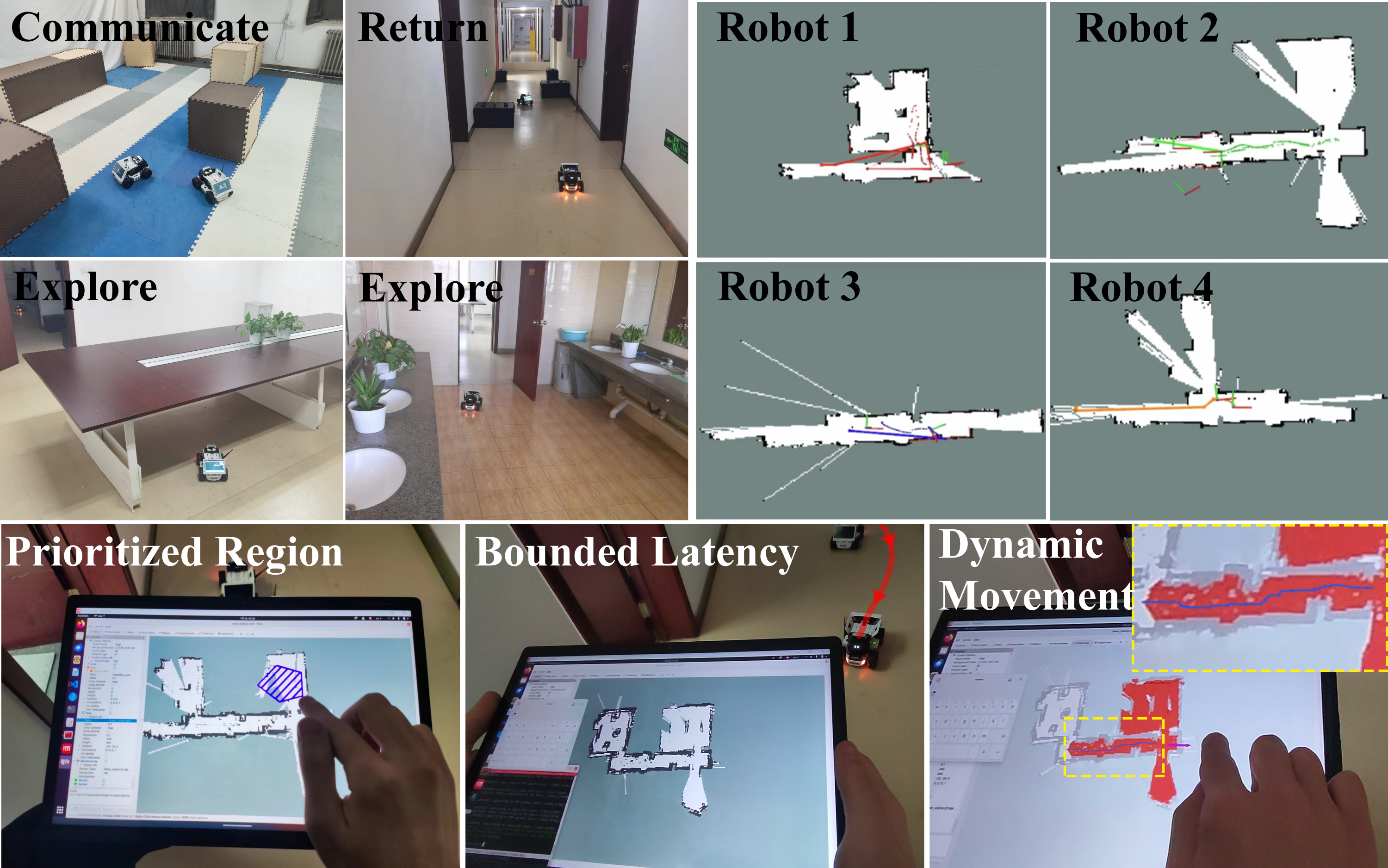}
  \vspace{-2mm}
  \caption{Illustration of the considered scenario.
    \textbf{Top-left}: $4$ UGVs switch among exploration, intermittent communication
    and return to the operator;
    \textbf{Top-right}: Different local maps at each robot,
    which are updated via pairwise communication;
    \textbf{Bottom}: three types of human requests:
    (i) Prioritized regions (in blue) to explore (\textbf{Left});
    (ii) Status update with bounded latency via return events (\textbf{Middle});
    (iii) Dynamic movement of the operator within the allowed region (in red) (\textbf{Right}).
  }
  \label{fig:overall}
  \vspace{-4mm}
\end{figure}

Furthermore, the role of the human operator and
the online interaction with the fleet during exploration
are {less} studied and mostly replaced by a static base for visualization.
However, under scarce communication, the operator might be completely
\emph{unaware} of the current progress of exploration
if none of the robots return to the operator
to relay such information.
Consequently, the operator can not supervise the exploration task,
e.g., checking the system status such as battery and liveness
of each robot;
or dynamically adjusting the priorities of the areas to be explored,
as shown in Fig.~\ref{fig:overall}.
Lastly, it is common that the operator might request to move
dynamically inside the explored area,
in which case to maintain the interactions during movement
is quite challenging without a global communication.

\begin{figure*}[t]
  \centering
  \includegraphics[width=0.85\linewidth]{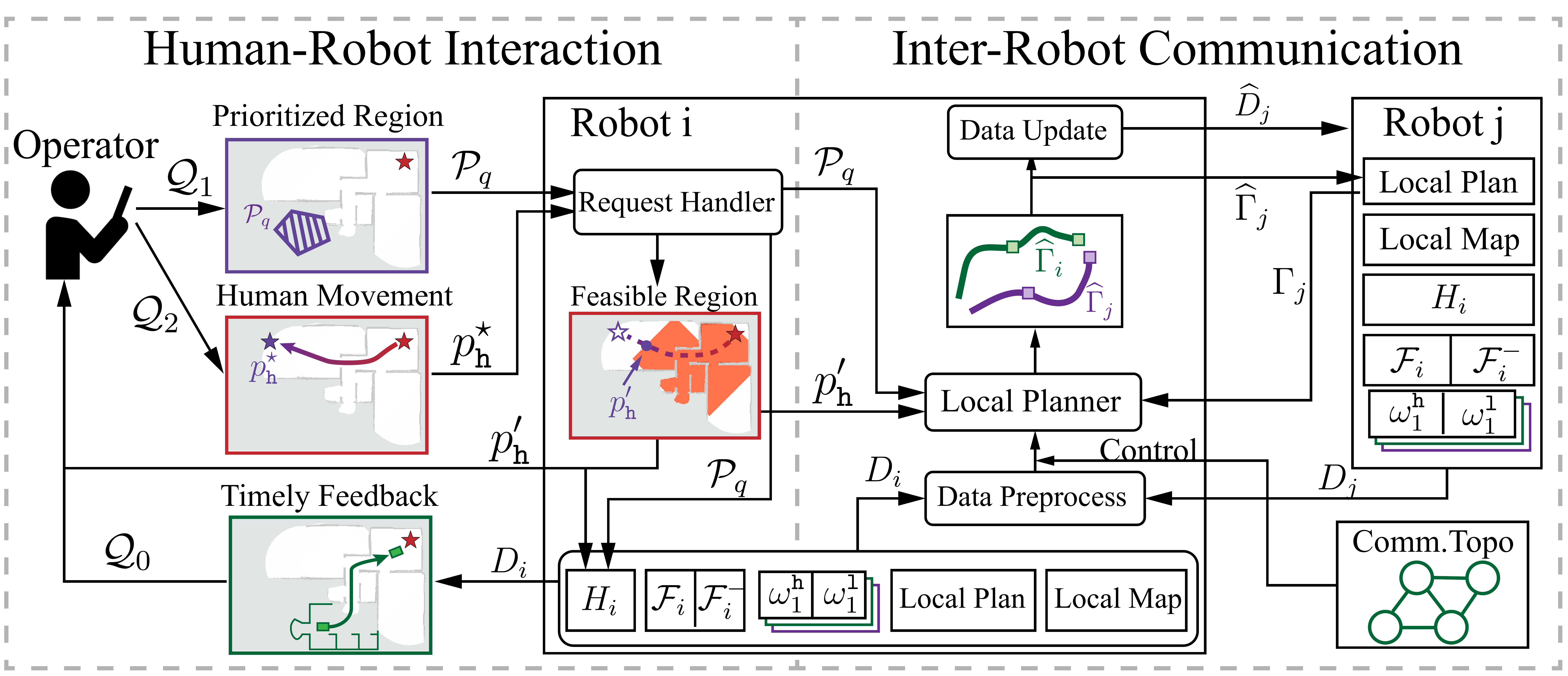}
  \caption{Illustration of the proposed framework (iHERO),
  which consists of the inter-robot and robot-operator communication
  protocol, the collaborative exploration strategy, and more importantly,
  the online adaptation module to operator requests.}
  \label{fig:framework}
  \vspace{-4mm}
\end{figure*}

\subsection{Related Work}\label{subsec:intro-related}
The problem of collaborative exploration has a long history in robotics.
The frontier-based method from~\citet{yamauchi1997frontier}
introduces an intuitive yet powerful metric for guiding the exploration.
Originally proposed for a single robot, it has been adapted to multi-robot teams
by assigning these frontiers to different robots for concurrent
exploration in different ways,
e.g., greedily to the nearest robot by distance in~\citet{yamauchi1999decentralized},
distributed and sequential auction in~\citet{hussein2014multi},
or the multi-vehicle routing for the optimal sequences in~\citet{zhou2023racer},
or based on the expected information gain in~\citet{colares2016next, patil2023graph, burgard2005coordinated}.
However, all the aforementioned work assume that all robots can
exchange information via wireless communication instantly at all time.
Thus, all robots always have access to \emph{the same global map} and the same set of frontiers.
This is often impractical or infeasible without a wireless network already installed,
especially for subterranean or indoor structures,
where the inter-robot communication is poor and limited in range.
Consequently, these approaches are not suitable anymore for communication-constrained unknown scenes.

To overcome this challenge of limited communication,
many recent work can be found that combines the planning
of inter-robot communication and autonomous exploration.
The work in~\citet{klaesson2020planning} reports the systematic solution
in the DARPA subterranean challenge where an integer linear program (ILP)
is formulated over the mobility-communication network.
Different intermittent communication strategies have been proposed
for different performance metrics, e.g.,
the ``Zonal and Snearkernet'' relay algorithm in~\citet{vaquero2018approach},
the four-state strategy in~\citet{cesare2015multi},
the distributed assignment of relay tasks for static goals in~\citet{marchukov2019fast},
and a centralized integer program to optimize rendezvous points in~\citet{gao2022meeting}.
On the other hand, fully-connected networks at all time are enforced in~\citet{rooker2007multi}
and further in~\citet{pei2010coordinated} with a lower-bounded bandwidth.
Droppable radios are utilized in~\citet{saboia2022achord} as
extended communication relays between robots and a static base station.
Their common objective is to maximize the exploration efficiency,
however without considering different online interactions with the operator.

Indeed, the human operator plays \emph{an indispensable role}
for the operation of robotic fleets, despite their autonomy.
In many aforementioned scenarios,
the operator should not only be aware of their operation status
online, but also directly supervise certain procedures whenever necessary,
e.g., prioritized regions to explore, or desired new position of the operator.
Almost all related work neglect this aspect and assume a static base station,
see, e.g.,~\citet{klaesson2020planning, vaquero2018approach,
  marchukov2019fast, gao2022meeting, pei2010coordinated, saboia2022achord}.
This interplay brings numerous challenges to be addressed, i.e.,
how to ensure a timely exchange of information between the operator and the fleet;
how to accommodate different requests;
and how to cope with the dynamic movement of the operator.

\subsection{Our Method}\label{subsec:intro-our}
This work proposes an interactive human-oriented online planning framework
for the collaborative exploration of unknown scenes (iHERO),
via a fleet of mobile robots under scarce communication.
As illustrated in Fig.~\ref{fig:framework},
explicit interactive requests from the operator are formulated
including timely status update, prioritized region during exploration,
and dynamic movements.
The proposed solution is a fully distributed coordination strategy for
both collaborative exploration and intermittent communication.
It guarantees an upper bound on the maximum latency of the
information gathered by any robot to the operator at all time.
Moreover, it is fully adaptive to the online requests
by adjusting the exploration and communication events accordingly,
subject to local communication and tight latency constraints.
Lastly, it provides real-time feedback to the operator regarding the allowed region of movement.
Extensive human-in-the-loop simulations and hardware experiments
are conducted for robotic fleets within office scenes and subterranean caves.

Main contribution of this work is threefold:
(i) the novel problem formulation of human-oriented collaborative exploration
and supervision under scarce communication, supporting three types of requests.
To the best of our knowledge, this problem \emph{has not been addressed} in related work;
(ii) the distributed coordination strategy for
both fast exploration
and inter-robot-operator communication,
while ensuring a timely response to online requests;
(iii) the numerical validation from extensive simulation and hardware experiments
for its applicability to real-world scenarios.

\section{Problem Description}\label{sec:problem}

\subsection{Workspace and Robots}\label{subsec:ws}
Consider a 2D workspace~$\mathcal{A}\subset \mathbb{R}^2$,
of which its \emph{map} including the boundary, freespace
and obstacles are all unknown.
To obtain a complete map of~$\mathcal{A}$,
a human operator deploys a team of robots
such as UAVs and UGVs around the location~$p_0\in \mathcal{A}$,
denoted by~$\mathcal{N}\triangleq\{1,\cdots,N\}$.
Each robot~$i\in \mathcal{N}$ is equipped with various sensors
such as~IMU, Lidar and RGBD cameras; and actuators to move around the workspace.
Moreover, each robot is assumed to be capable of simultaneous localization and mapping (SLAM)
with collision avoidance,
i.e.,
\begin{equation}\label{eq:slam}
  (M^+_i,\, p^+_i,\, \mathbf{p}_{g}) \triangleq \texttt{SLAM}(p_i,\, p_g,\, M_i,\,\mathcal{A}),
\end{equation}
where~$\texttt{SLAM}(\cdot)$ is a generic SLAM module;
$M_i(t)\subset \mathcal{A}$ is the \emph{local} map of robot~$i\in \mathcal{N}$ at time~$t\geq 0$;
$p_i(t)\in M_i$ is the 2D pose of robot~$i$;
$p_g \in M_i$ is a suitable goal pose within~$M_i$;
$\mathbf{p}_{g}\subset M_i$ is the collision-free path from~$p_i(t)$ to~$p_g$
within~$M_i$;
$M^+_i(t') \subset \mathcal{A}$ is the updated map
after traversing the path at time~$t'>t$;
and~$p^+_i(t')$ is the updated pose at time~$t'$.

\begin{remark}\label{rm:map-model}
The exact representation of the map depends on the SLAM algorithm,
e.g., the occupancy grid map from~\citet{thrun2002probabilistic},
the octomap from~\citet{hornung2013octomap},
or even the metric-semantic map in~\citet{tian2022kimera}.
\hfill $\blacksquare$
\end{remark}

\begin{figure}[t]
  \centering
  \includegraphics[width=0.95\linewidth]{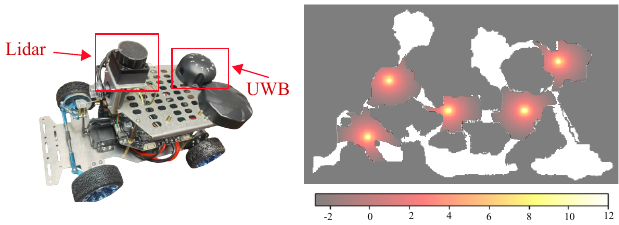}
  \vspace{-2mm}
  \caption{\textbf{Left}: UGV equipped with Lidar and UWB-based
  communication unit;
    \textbf{Right}: the communication quality
    when~$5$ robots are scattered
    within a subterranean cave. }
  \label{fig:uwb}
  \vspace{-4mm}
\end{figure}

\subsection{Inter-robot Communication}\label{subsec:com}
Each robot is equipped with a wireless communication unit
for inter-robot information exchange, e.g., WiFi and ultra-wideband (UWB).
In particular, robot~$i\in \mathcal{N}$ can communicate with robot~$j\in \mathcal{N}$
at time~$t>0$,
if the communication quality (such as SNR) is above a certain threshold, i.e.,
\begin{equation}\label{eq:com}
  (D^+_i, D^+_j) \triangleq \texttt{Com}\big(D_i,D_j\big),
  \,\text{ if }\, \texttt{Qual}(p_i,p_j,\mathcal{A}) > \underline{\delta};
\end{equation}
where~$\texttt{Com}(\cdot)$ is the communication module of robots~$i$ and~$j$;
$D_i(t)\in \mathcal{D}$ is the {local} data stored at robot~$i$ at time~$t\geq 0$,
similar for robot~$j$;
$\mathcal{D}$ is a pre-defined data structure, which may include the map data,
video/audio stream, coordination plans, requests from the human operator
as explained in the sequel;
$\texttt{Qual}(\cdot)$ evaluates the communication quality between robots~$i,j$
given their positions and the workspace layout;
$\underline{\delta}>0$ is a pre-specified lower bound on the quality;
the outputs~$D^+_i(t), D^+_j(t)$ are the updated data after communication.
By default, $(M_i,\, p_i)\in D_i$ holds for all robots.
An illustration of the UWB-based communication is shown in Fig.~\ref{fig:uwb},
which is subject to both range and line-of-sight (LOS) constraints.

\subsection{Human-robot Interaction and Online Requests}\label{subsec:robot-human}
Similarly, the human operator can receive data from any robot
via the same communication module in~\eqref{eq:com}.
Namely, the local data and map stored at the operator are denoted
by~$D_\texttt{h}(t),M_\texttt{h}(t)$, respectively.
More importantly, the operator has the following \textbf{three} types of
{requests} for the fleet:
(i) the operator should obtain the local maps of {all} robots with a latency
smaller than a given threshold~$T_{\texttt{h}}>0$ at all time, i.e.,
\begin{equation}\label{eq:frequency}
\textstyle{\bigcup_i} M_i(t) \subset M_\texttt{h}(t+T_{\texttt{h}}), \; \forall t \geq 0;
\end{equation}
which means that the explored map by any robot at time~$t$ should be known
to the operator latest by time~$t+T_{\texttt{h}}$.
This request is denoted by~$Q_0$, which is activated at all time;
(ii) the locations that should be prioritized during exploration,
denoted by~$Q_1\subset \mathcal{A}$.
It is fulfilled once the area in~$Q_1$ is explored;
(iii) the desired next position of the operator,
denoted by~$Q_2\triangleq p_\texttt{h}^{\star}\in \mathcal{A}$,
which is fulfilled after the operator reaches it.
  As described in the sequel, the~$Q_2$ request might not be feasible
  due to the constraint~\eqref{eq:frequency}.
  Thus, intermediate waypoints that are feasible would be provided for
  the operator to choose as an interactive process.

Thus, the behavior of each robot is fully determined
by its timed sequence of navigation and communication events, i.e.,
\begin{equation}\label{eq:plan}
\Gamma_i\triangleq c^0_i\, \mathbf{p}^0_i\, c^1_i \, \mathbf{p}^1_i\,c^2_i \, \cdots,
\end{equation}
where~$c^k_i$ is a communication event defined by
the communication location, the other robot or the operator,
and the exchanged data as defined in~\eqref{eq:com};
the navigation path~$\mathbf{p}^k_i\subset \mathcal{A}$ is the waypoints
that robot~$i$ has visited between the events;
and~$k>0$ records the number.
For brevity, denote by~$\{\Gamma_i\} \models Q_{0,1,2}$
if the behaviors of the fleet fulfill the requirements of these requests collaboratively.

\subsection{Problem Statement}\label{subsec:problem}
The overall objective is to design the collaborative exploration and communication strategy,
such that the total time for the operator to obtain the complete map is minimized.
Additionally, the online requests from the operator should be accommodated.
Thus, it is formalized as a constrained optimization over the jointed plan of all robots, i.e.,
\begin{subequations} \label{eq:problem}
  \begin{align}
    &\mathop{\mathbf{min}}\limits_{\{\Gamma_i\},\, \overline{T}}\;  \overline{T} \notag\\
    \textbf{s.t.}\quad & \mathcal{A} \subseteq M_\texttt{h}(\overline{T}); \label{subeq:terminal}\\
    &\eqref{eq:slam}-\eqref{eq:com},\; \forall \Gamma_i; \label{subeq:dynamics}\\
    &\{\Gamma_i\} \models Q_\ell,\;\forall Q_\ell\in \mathcal{Q}; \label{subeq:requests}
  \end{align}
\end{subequations}
where~$\overline{T}>0$ is the time when
the complete map~$\mathcal{A}$ is known to the operator in~\eqref{subeq:terminal};
\eqref{subeq:dynamics} restricts the behaviors of the fleet as described;
and~\eqref{subeq:requests} enforces that the set of all requests~$\mathcal{Q}$
from the operator are fulfilled.
  Note that~\eqref{subeq:terminal} is overall goal of exploration
  which is independent of the underlying system, while~\eqref{subeq:dynamics}
  is system-dependent.
  For fully-connected system without communication constraints,
the constraint~\eqref{subeq:terminal} is often replaced by~$\mathcal{A} \subseteq \bigcup_i M_i(\overline{T})$ as the union of all local maps.

\begin{remark}\label{rm:prob}
  The latency~$T_{\texttt{h}}$ in~$Q_0$ could be modified online
  by the operator according to the exploration progress,
  e.g., it can be increased gradually as the fleets move further away from the operator
  and the newly-explored area decreases.
  For simplicity, the subsequent analyses are conducted
  for a static~$T_{\texttt{h}}$ first,
  which is then extended in Sec.~\ref{subsubsec:dynamic-Th}.
\hfill $\blacksquare$
\end{remark}

\section{Proposed Solution}\label{sec:solution}
The solution contains three main components:
(i) the intermittent communication protocol as the fundamental
building block;
(ii) the strategy
for simultaneous exploration and intermittent communication
to fulfill the~$Q_0$ and $Q_1$ requests;
and (iii) the interactive protocol to fulfill the~$Q_2$ request,
where the operator is offered intermediate
waypoints before the goal.

\subsection{Intermittent Communication Protocol}\label{subsec:comm-protocol}
To ensure that the information can be propagated among the team and to the operator,
the robots are required to meet and communicate via the protocol
of {intermittent} communication during exploration~\citet{guo2018multirobot,kantaros2019temporal}.
To begin with, a communication topology is defined
as an undirected graph~$\mathcal{G}\triangleq (\mathcal{N},\mathcal{E})$,
where~$\mathcal{E}\subset \mathcal{N}\times  \mathcal{N}$.
It satisfies two conditions:
(i) $\mathcal{G}$ is connected;
(ii) two robots~$i,j$ can only communicate at time~$t$
if~$(i,j)\in \mathcal{E}$ and $\texttt{Qual}(\cdot)$ in~\eqref{eq:com} holds.
Given a topology~$\mathcal{G}$,
the protocol is as follows:
initially, all robots are in close proximity but
each robot only communicates with its neighbors in~$\mathcal{G}$;
during each communication, two robots
determine locally the ``time and place'' for
their \emph{next} communication event,
and whether one robot needs to return
to the operator {before} they meet (called a \emph{return event});
then they depart and do not communicate until the next communication event.
This procedure repeats indefinitely until termination,
as shown in Fig.~\ref{fig:comm-protocol}.

\begin{figure}[t]
  \centering
  \includegraphics[width=0.85\linewidth]{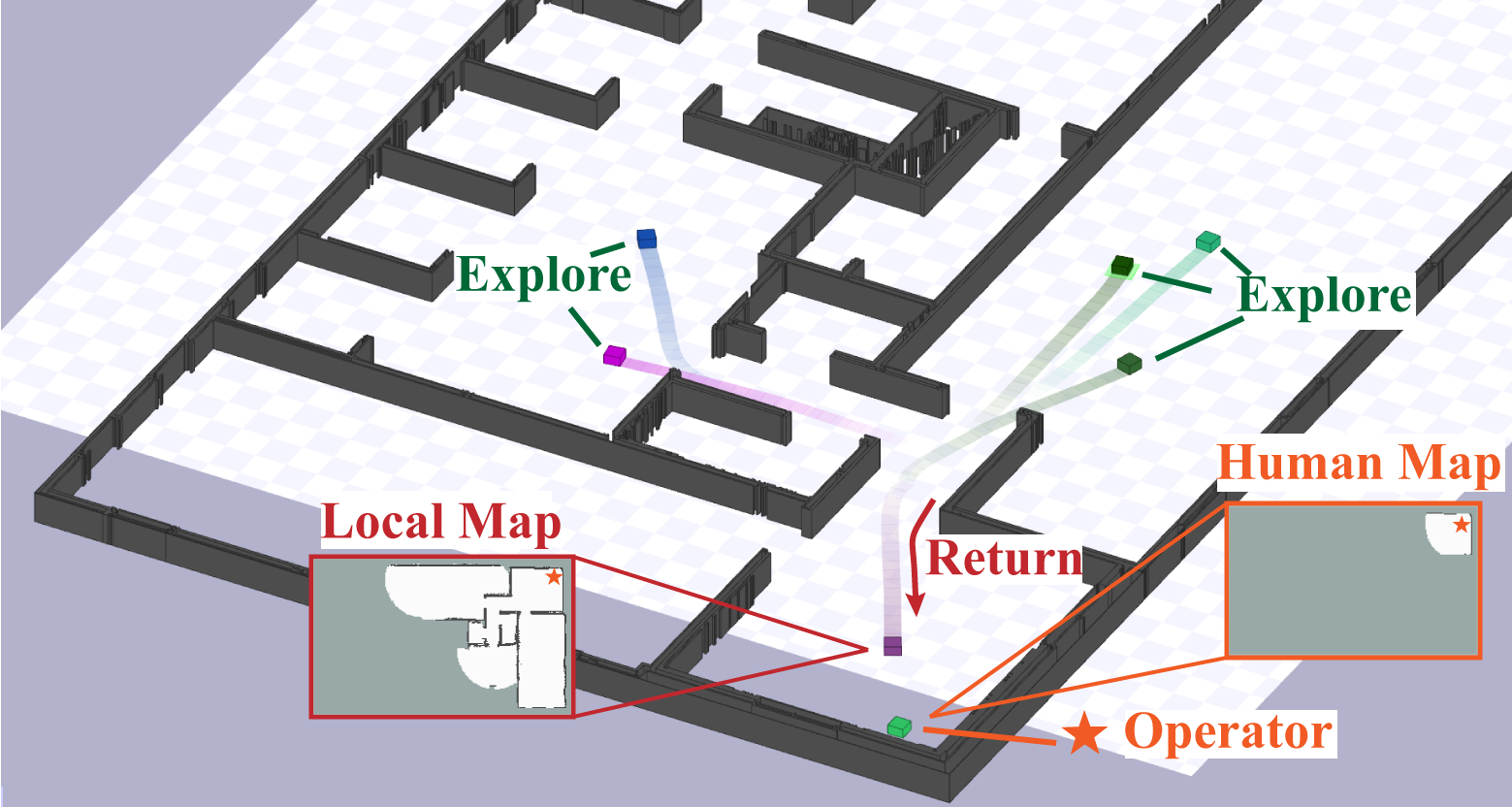}
  \caption{Illustration of the intermittent communication
  protocol, including pair-wise communication and return events.}
  \vspace{-4mm}
  \label{fig:comm-protocol}
\end{figure}

\begin{definition}[Communication Rounds]\label{def:rounds}
Under this protocol, all communication events can be sorted in time as rounds:
\begin{equation}\label{eq:round}
  \mathcal{C}\triangleq C_1C_2\cdots C_r\cdots,
\end{equation}
where~$C_r\triangleq c_{n_1 n'_1}\, c_{n_2 n'_2}\,\cdots c_{n_{K_r} n'_{K_r}}$
is the $r$-th round of communication for~$r\geq 1$;
$c_{n_kn'_k}\triangleq (t,\,p,\,D)$ is the communication {event}
between robots~$(n_k,n'_k)\in \mathcal{E}$, representing the time,
location and the exchanged data, respectively;
and~$c_{n_kn'_k}\in C_r$, $\forall (n_k,\, n'_k)\in \mathcal{E}$
and $\forall r\geq 1$.
    \hfill $\blacksquare$
\end{definition}
The last condition states that all neighbors in~$\mathcal{E}$
should have communicated \emph{at least} once for each round~$C_r\in \mathcal{C}$.
During each round of communication,
one or several return events are determined online,
yielding the communication events between the robots and the operator.
Similarly,
let~$C_\texttt{h}\triangleq c^\texttt{h}_{i_1}\,c^\texttt{h}_{i_2}\,
\cdots c^\texttt{h}_{i_k}\,\cdots$
denote all the return events sorted in time,
where~$c^\texttt{h}_{i_k}\triangleq (t_k,\, D_{i_k},\, T^\texttt{h}_k,\, Q_k)$
is the communication event between robot~$i_k$ and the operator;
$t_k$ represents the time of the $k$-th return event;
$D_{i_k}$ denotes the data propagated to operator from robot~$i_k$;
$T^{\texttt{h}}_k\triangleq t^{\texttt{h}}_{k,1}t^{\texttt{h}}_{k,2}\cdots t^{\texttt{h}}_{k,N}$
is the time stamps of the data from each robot contained in~$D_{i_k}$;
and~$Q_k$ is the request from the operator.
Each time a robot returns,
the operator receives data, updates its own knowledge accordingly,
and then sends requests.
Given these return events,
the data stored by the operator at time~$t$
is given by~$D_\texttt{h}(t)\triangleq \bigcup_{t_k\leq t}D_{i_k}$.
Denote by $T_\texttt{h}(t)\triangleq t^\texttt{h}_1(t)t^\texttt{h}_2(t)\cdots t^\texttt{h}_N(t)$
the corresponding timestamps of $D_\texttt{h}(t)$.
It can be shown that if the condition:
\begin{equation}\label{eq:delta-r}
  t-t^\texttt{h}_n(t)\leq T_\texttt{h},
\end{equation}
holds, $\forall t\geq 0, \forall n\in \mathcal{N}$,
then the constraint~\eqref{eq:frequency} is fulfilled.
Moreover, the constraint~\eqref{eq:delta-r}
can be fulfilled by putting constraints on each return event,
hence enabling distributed coordination for a global constraint,
as stated below.

\begin{lemma}\label{prop:replace-Ts}
The constraint~\eqref{eq:delta-r} is fulfilled, if
both (i) $t_{k+1}\leq T_{\texttt{h}} + \chi^{\texttt{h}}_k$;
 and (ii) $\chi^{\texttt{h}}_k < \chi^{\texttt{h}}_{k+1}$,
 hold for each~$k\geq 0$,
 where $\chi^{\texttt{h}}_k\triangleq \mathbf{min}_n\,\big\{T^\texttt{h}_k[n]\big\}$.
\end{lemma}
\begin{proof}
  See Appendix in Sec.~\ref{sec:app}.
\end{proof}

The first condition above indicates that each return event should occur
before the largest latency exceeds the given bound,
estimated by the minimal timestamp of the previous return events plus $T_\texttt{h}$;
the second condition above ensures that the minimal stamp
of each return event is strictly increasing,
which in turn guarantees a monotonic decrease of the largest latency
within the fleet.
These two conditions combined ensure that the latency of each robot
is bounded by $T_\texttt{h}$ at all time.


\subsection{Simultaneous Exploration
  and Intermit. Communication}
\label{subsec:spread}
Following the above protocol,
this section describes how the communication events in~\eqref{eq:round}
above are optimized online in a distributed way for the $Q_0$ and $Q_1$ requests.

\subsubsection{Distributed Frontier-based Exploration}\label{subsec:frontiers}
The frontier-based method in~\citet{yamauchi1997frontier}
allows a robot to explore the environment
by repetitively reaching the frontiers on the boundaries between
the explored and unexplored areas within its local map.
Briefly speaking,
given the local map~$M_i(t)$ of robot~$i\in \mathcal{N}$ at time~$t>0$,
these boundaries can be identified via a Breadth-First-Search (BFS),
which are then clustered to a few frontiers by various
metrics from~\citet{holz2010evaluating}.
Denote by~$\mathcal{F}_i(t)=\{f_k\}$ the set of frontiers,
where each frontier~$f_k\in \partial M_i$ is on the boundary of explored areas~$\partial M_i$.
Thus, $\mathcal{F}_i(t)$ is empty if~$M_i(t)$ is fully explored.

Following the protocol of intermittent communication,
consider that robots~$i$ and $j$ communicate during the
the event~$c_{ij}\in C_r$ at time~$t$.
With a slight abuse of notation,
the current local \emph{plan} of robot~$i$ is given by:
\begin{equation}\label{eq:local-plan}
\Gamma_i(t)\triangleq c_{ij}\mathbf{p}^0_i c^1_i \mathbf{p}^1_ic^2_i\cdots \mathbf{p}^{K_i}_ic^{K_i}_i,
\end{equation}
where~$c_{ij}$ is the current event;
$c^k_i$ is the planned sequence of future events;
$\mathbf{p}^k_i$ is the planned path between these events including visiting frontiers;
and $K_i=|\mathcal{N}_i|$ is the number of neighbors for
robot~$i$ in~$\mathcal{G}$.
Note that the time and place for the events~$\{c^k_i\}\subset \Gamma_i(t)$
are already confirmed, thus \emph{should not} be altered.
However, the intermediate paths~$\{\mathbf{p}^k_i\}$ can be modified
by choosing a different path to visit more frontiers,
as shown in Fig.~\ref{fig:local-optimization}.
More specifically,
denote by~$\widehat{\Gamma}_i(t)\triangleq
\widehat{\mathbf{p}}^0_i c^1_i \widehat{\mathbf{p}}^1_i
c^2_i\cdots \widehat{\mathbf{p}}^{K_i}_ic^{K_i}_i\widehat{\mathbf{p}}_{ij}\widehat{c}_{ij}$
the \emph{revised} plan of robot~$i$,
where~$\{\widehat{\mathbf{p}}^k_i\}$ are the updated paths
and~$\widehat{\mathbf{p}}_{ij}$ is the newly-added path
before the next  event~$\widehat{c}_{ij}$.
Thus, the sub-problem of local coordination can be formulated as follows.
\begin{figure}[t]
	\centering
	\includegraphics[width=0.9\linewidth]{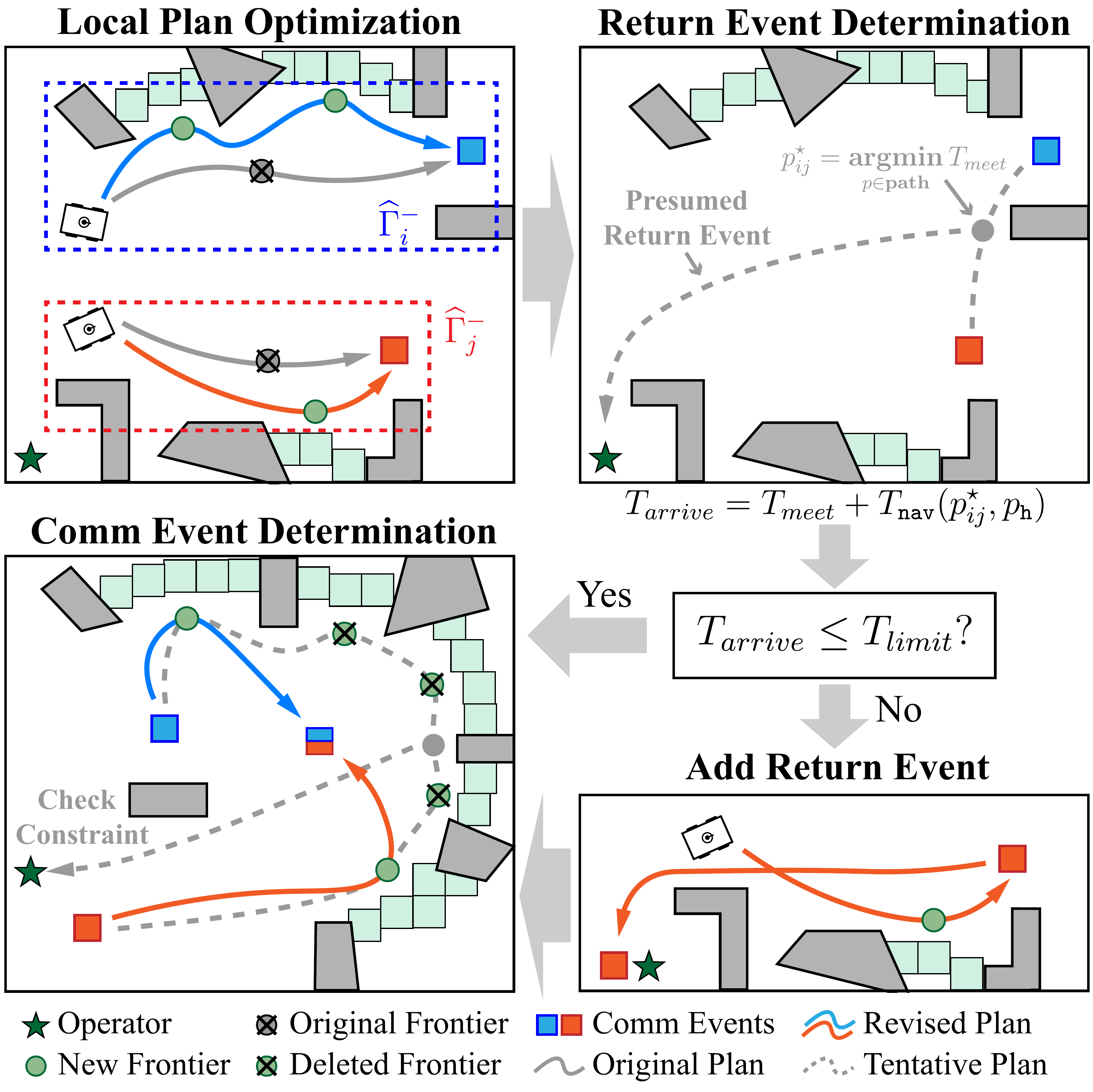}
	\caption{Illustration of the pairwise coordination strategy
  during intermittent communication event.
  }
	\label{fig:local-optimization}
	\vspace{-7mm}
\end{figure}

\begin{problem}\label{prob:local-plan}
Given the local data $(\Gamma_i,\,M_i)$ and $(\Gamma_j,\,M_j)$,
the objective is to synthesize the revised plans~$\widehat{\Gamma}_i$, $\widehat{\Gamma}_j$
such that: (i) the total number of frontiers that are planned
to be visited by robots~$i,j$ is maximized;
and (ii) the largest latency~$\delta_r$ in~\eqref{eq:delta-r}
for the current round is bounded by~$T_{\texttt{h}}$.
\hfill $\blacksquare$
\end{problem}

  Note that the first objective is defined locally for
  robots~$i,j$, whereas the second objective is a global constraint.
  The remaining part describes how it can be solved in a distributed way
  via intermittent communication and local optimization.

\subsubsection{Optimization of Local Plans}\label{subsec:coordination}
To begin with, robot~$i$ executes event~$c_{ij}$ in~\eqref{eq:local-plan} by meeting
with robot~$j$ around time~$t_{ij}$ and at location~$p_{ij}$.
The exchanged {data} from robot~$i$ is given by:
\begin{equation}\label{eq:exchanged-data}
D_i(t) \triangleq (M_i,\,\Gamma_i,\,\mathcal{F}'_i,\,\Omega_i, H_i),
\end{equation}
where~$M_i,\Gamma_i$ are defined earlier;
$\mathcal{F}'_i\triangleq \mathcal{F}_i\cup \mathcal{F}^-_i$
is the set of frontiers known to robot~$i$,
with~$\mathcal{F}_i$ being the frontiers assigned to robot~$i$
and~$\mathcal{F}^-_i$ the frontiers known to robot~$i$ but assigned
to~\emph{other} robots;
$\Omega_i\triangleq (\Omega_i^\texttt{l},\,\Omega_i^\texttt{h})$
is a tuple of two vectors with length~$N$,
i.e., $\Omega^\texttt{l}_i\triangleq \omega^\texttt{l}_1\omega^\texttt{l}_2
\cdots \omega^\texttt{l}_N$
is the estimated timestamps of the data from robot~$n\in \mathcal{N}$
that is known to robot~$i$ after its last communication event $c^{K_i}_i$;
and $\Omega^\texttt{h}_i\triangleq \omega^\texttt{h}_1\omega^\texttt{h}_2
\cdots \omega^\texttt{h}_N$
is the {estimated} timestamps of data from robot~$n\in \mathcal{N}$
that is known to the operator at time~$t^{K_i}_i$;
$H_i\triangleq \{(\mathcal{P}_q,\,p^{\star}_\texttt{h},\,t_\texttt{h})\}$
contains the requests from the operator,
including the prioritized region~$\mathcal{P}_q\subseteq {Q}_1$,
the desired position~$p^{\star}_\texttt{h}\in {Q}_2$,
and~$t_{\texttt{h}}$ is the timestamp for this request.
Note that~$\Omega^\texttt{l}_i$ and $\Omega^\texttt{h}_i$ are initialized as zeros
and updated each time when a communication event is confirmed.
The local data~$D_j$ of robot~$j$ is defined analogously.

After~$D_i$ and~$D_j$ are exchanged,
these data is processed by robot~$i$ as follows.
First, the merged map of~$M_i$ and~$M_j$ is computed as its new local map,
i.e., $M_{ij}\triangleq \texttt{merge}(M_i,M_j)$.
Then, the frontiers~$\mathcal{F}_{ij}$ associated with the merged~$M_{ij}$ is
computed, and the complete set of frontiers that can be visited is given by
$\widehat{\mathcal{F}} \triangleq  \mathcal{F}_{ij} \backslash
\big((\mathcal{F}^-_i \backslash \mathcal{F}_j)
\cup (\mathcal{F}^-_j \backslash \mathcal{F}_i)\big)$,
which excludes the frontiers in~$\mathcal{F}_{ij}$ that have been assigned
to other robots in the previous events.
Next, a new vector of timestamps~$\Omega^{\texttt{h}}_{ij}$ is calculated by
taking the maximum of $\Omega^{\texttt{h}}_i$ and $\Omega^{\texttt{h}}_j$
to derive the estimated stamps of each robot known to operator, i.e.,
\begin{equation}\label{eq:estimate-Th}
  \Omega^\texttt{h}_{ij}[n] \triangleq
  \textbf{max}\Big\{\Omega^\texttt{h}_i[n],\, \Omega^\texttt{h}_j[n]\Big\},
\end{equation}
where~$\Omega^{\texttt{h}}_{i}[n]=\omega^h_n$ is the timestamp of the data
from robot~$n\in \mathcal{N}$ as defined in~\eqref{eq:exchanged-data}.
It serves as an estimation of the timestamps~$T_{\texttt{h}}(t)$ for the operator in this round,
which is essential to determine whether a return event is needed later.
Lastly,
the requests~$H_i$ are merged with $H_j$ by the timestamps.

\begin{remark}\label{remark:merge}
The local map merging described above can be done
efficiently based on~\cite{andre2014coordinated},
often with a common coordinate system that is agreed-upon beforehand,
or via map fusion and alignment during execution
as proposed in~\cite{lajoie2020door,tian2022kimera}.
Tackling uncertainty and misalignment
when merging local maps with varying accuracy
is outside the scope of this work.
\hfill $\blacksquare$
\end{remark}

Given~$\widehat{\mathcal{F}}$ above,
a planning algorithm consists of three main steps:
(i) the frontiers in~$\widehat{\mathcal{F}}$ are divided into two clusters
$\widehat{\mathcal{F}}_i, \widehat{\mathcal{F}}_j$
associated with either robot~$i$ or $j$,
e.g., by comparing the minimum distance from~$f\in \widehat{\mathcal{F}}$ to
the confirmed communication locations in~$\Gamma_i$ and $\Gamma_j$;
(ii) a traveling salesman problem with time windows (TSPTW) is formulated
for each robot locally to optimize the sequence of frontiers to visit between
communication events,
i.e., the paths~$\{\widehat{\mathbf{p}}^k_{i}\}$ for robot~$i$ to
visit~$\widehat{\mathcal{F}}_i$
(the same for robot~$j$).
Detailed modeling is omitted here due to limited space.
The resulting preliminary plan is given by~$\widehat{\Gamma}'_i=
\widehat{\mathbf{p}}^0_i c^1_i \widehat{\mathbf{p}}^1_i
c^2_i\cdots \widehat{\mathbf{p}}^{K_i}_ic^{K_i}_i\mathbf{p}^{K_i+1}_{i}$,
which is \emph{not} yet the final revised plan~$\widehat{\Gamma}_i$
as the communication event~$\widehat{c}_{ij}$ has not been determined;
(iii) the next communication event~$\widehat{c}_{ij}$
is optimized to fulfill~\eqref{eq:delta-r},
given~$\widehat{\Gamma}'_i$ and~$\widehat{\Gamma}'_j$.

\begin{remark}\label{remark:mvrp}
  The first two steps above can \emph{not} be replaced by formulating and
  solving a multiple-vehicle routing problem with time window (MVRPTW).
  This is due to the fact that the communication events are \emph{already assigned}
  to respective robots and only the frontiers between these events can be shuffled
  subject to the time windows.
  \hfill $\blacksquare$
\end{remark}

\subsubsection{Next Communication Event}\label{subsec:communication}
To ensure that the latency~$\delta_r \leq T_{\texttt{h}}$,
an additional constraint regarding
the next communication event~$\widehat{c}_{ij}$ is added, i.e.,
\begin{equation}\label{eq:meet-constraint}
  \widehat{t}_{ij}+T_i^{\texttt{nav}}(\widehat{p}_{ij},\, p_{\texttt{h}})
  \leq T_{\texttt{h}}
  +\textbf{min}_n\, \big\{\Omega^{\texttt{h}}_{ij}[n]\big\},
\end{equation}
where~$\widehat{t}_{ij},\, \widehat{p}_{ij}$ are the expected
meeting time and location, respectively;
$p_{\texttt{h}}$ is the current location of the operator;
$T_i^{\texttt{nav}}(p_1,\, p_2)$ is the estimated duration
for robot~$i$ to travel from location~$p_1$ to~$p_2$ within its map~$M_i$,
(which is set to $\infty$ if~$p_2$ is un-reachable);
$\Omega^\texttt{h}_{ij}$ is the vector of timestamps
as defined in~\eqref{eq:estimate-Th}.
It indicates that if robot~$i$ meets the operator directly
after meeting robot~$j$,
the operator can receive the data from each robot~$n\in \mathcal{N}$
with a latency less than~$T_{\texttt{h}}$.
Note that~$\Omega^\texttt{h}_{ij}$ in constraint~\eqref{eq:meet-constraint}
is updated according to the last confirmed communication event
in~$\widehat{\Gamma}'_i$ and~$\widehat{\Gamma}'_j$.
Moreover, since the meeting events in the preliminary
plan~$\widehat{\Gamma}'_i$ are already fixed,
it can be divided into two segments:
$\widehat{\Gamma}^-_i \triangleq \widehat{\mathbf{p}}^0_i c^1_i
\cdots \widehat{\mathbf{p}}^{K_i}_ic^{K_i}_i$
and $\mathbf{p}^{K_i+1}_{i}$(the same for robot~$j$).
In the final plan,
the segment~$\widehat{\Gamma}^-_i$ is reserved,
while the remaining frontiers, denoted
by~$\widehat{\mathcal{F}}_{ij}$,
are re-assigned as intermediate waypoints between the
events~$c^{K_i}_{i}$ and $\widehat{c}_{ij}$ for robot~$i$
(between~$c^{K_j}_{j}$ and $\widehat{c}_{ij}$ for robot~$j$).
Thus, the final plan is given
by~$\widehat{\Gamma}_i \triangleq \widehat{\Gamma}^-_i + \widehat{\Gamma}^+_i$,
where~$\widehat{\Gamma}^+_i \triangleq \widehat{\mathbf{p}}^{K_i+1}_{i}
\widehat{c}_{ij}$ and $\widehat{\Gamma}^+_j \triangleq \widehat{\mathbf{p}}^{K_j+1}_{j}
\widehat{c}_{ij}$ are the tailing segments to be optimized as stated below.

\begin{problem}\label{prob:next-event}
Given~$\widehat{\Gamma}^-_i$, $\widehat{\Gamma}^-_j$
and $\widehat{\mathcal{F}}_{ij}$,
$\widehat{\Gamma}^+_i,\, \widehat{\Gamma}^+_j$ should be chosen
such that:
(i) the number of frontiers contained
in~$\widehat{\Gamma}^+_i, \widehat{\Gamma}^+_j$ is maximized;
and (ii) condition~\eqref{eq:meet-constraint} holds for $\widehat{c}_{ij}$.
\hfill $\blacksquare$
\end{problem}

The proposed solution is summarized in Alg.~\ref{alg:opt-com}.
It consists of two main parts:
(i) to determine whether robot~$i$ (or robot~$j$)
should return to the operator as a {return event},
\emph{after} its latest confirmed communication event~${c}^{K_i}_i$
(or~${c}^{K_j}_j$)
and \emph{before} the communication event~$\widehat{c}_{ij}$;
(ii) to optimize the next communication event~$\widehat{c}_{ij}$
given the return event.

\begin{algorithm}[t]
	\caption{Optimize Next Communication Event}
  \label{alg:opt-com}
	\LinesNumbered
        \SetKwInOut{Input}{Input}
        \SetKwInOut{Output}{Output}
\Input{$\widehat{\Gamma}^-_i$, $\widehat{\Gamma}^-_j$,
  $\widehat{\mathcal{F}}_{ij}$.}
\Output{$\widehat{\Gamma}^+_i$, $\widehat{\Gamma}^+_j$.}
\tcc{\textbf{Return event}}
  $T_{\texttt{lim}} \leftarrow T_\texttt{h}+\textbf{min}_n\, \big\{\Omega^{\texttt{h}}_{ij}[n]\big\}$\;
  $\overline{\mathbf{p}}_{ij} \leftarrow \texttt{GenPath}(p^{K_i}_i,\, p^{K_j}_j)$\;\label{alg_line:gen_path}
  $c^\star_{ij} \leftarrow \texttt{SelComm}(c^{K_i}_i,\, c^{K_j}_j,\, \overline{\mathbf{p}}_{ij})$\;\label{alg_line:gen_comm}

  \If{not $\texttt{CheckConst}(c^\star_{ij},\, T_{\texttt{lim}})$}{\label{alg_line:need_return}
    $c^{K_j}_j \leftarrow \texttt{CommEvent}(p_{\texttt{h}},\, t^{K_j}_j+T_i^{\texttt{nav}}(p^{K_j}_j,p_{\texttt{h}}))$\;\label{alg_line:add_return}
    $\Omega^{\texttt{h}}_{ij} \leftarrow \texttt{UpdTimeStp}(\Omega^{\texttt{h}}_{ij},\, \Omega^{\texttt{l}}_j)$\;\label{alg_line:update_w}
    $T_{\texttt{lim}} \leftarrow T_\texttt{h}+\textbf{min}_n\, \big\{\Omega^{\texttt{h}}_{ij}[n]\big\}$\;\label{alg_line:update_lim}
  }
  \tcc{\textbf{Next comm. event}}
  $\widetilde{\mathcal{F}}_{ij}= \widehat{\mathcal{F}}_{ij}$\;
  \While{not $\texttt{CheckConst}(\widehat{c}_{ij},\, T_{\texttt{lim}})$}{\label{alg_line:continue_itr}
    $\widehat{\mathbf{F}}_{ij} \leftarrow \texttt{TSP}(p^{K_i}_i,\, \widetilde{\mathcal{F}}_{ij},\, p^{K_j}_j)$\;\label{alg_line:tsp}
    $\widehat{{\mathbf{p}}}_{ij} \leftarrow \texttt{GenPath}(p^{K_i}_i,\, \widehat{\mathbf{F}}_{ij},\, p^{K_j}_j)$\;\label{alg_line:gen_path_tsp}
  $\widehat{c}_{ij} \leftarrow \texttt{SelComm}(c^{K_i}_i,\, c^{K_j}_j,\, \widehat{{\mathbf{p}}}_{ij})$\;\label{alg_line:gen_comm_tsp}
   $f^\star \leftarrow {\mathbf{argmax}}_{f \in \widetilde{\mathcal{F}}_{ij}}\{cost(f)\}$\;\label{alg_line:select_frontier}
   $\widetilde{\mathcal{F}}_{ij}\leftarrow \widetilde{\mathcal{F}}_{ij}\backslash \{f^\star\}$\;\label{alg_line:remove_frontier}
  }
  $\widehat{\mathbf{p}}^{K_i+1}_i,\, \widehat{\mathbf{p}}^{K_j+1}_j \leftarrow
  \texttt{Split}(\widehat{{\mathbf{p}}}_{ij},\, \widehat{c}_{ij})$\;\label{alg_line:split}
  $\widehat{\Gamma}^+_i \leftarrow \widehat{\mathbf{p}}^{K_i+1}_i + \widehat{c}_{ij}$\;
  $\widehat{\Gamma}^+_j \leftarrow \widehat{\mathbf{p}}^{K_j+1}_j + \widehat{c}_{ij}$\;\label{alg_line:return}
\end{algorithm}

\textbf{Return event}.
Since both~$c^{K_i}_{i}$ and~$c^{K_j}_{j}$ belong to the merged map~$M_i$,
an obstacle-free shortest path~$\overline{\mathbf{p}}_{ij}$ can be found
by a path planning algorithm from~$c^{K_i}_{i}$ to~$c^{K_j}_{j}$ (Line~\ref{alg_line:gen_path}).
Then, a point~$p^\star_{ij}$ within~$\overline{\mathbf{p}}_{ij}$
is determined by minimizing the meeting time if~$p^\star_{ij}$ is chosen as
the meeting location, i.e.,
\begin{equation}\label{eq:select-p}
  p^\star_{ij}=\mathop{\mathbf{argmin}}\limits_{p \in \overline{\mathbf{p}}_{ij}}
  \left\{\mathop{\mathbf{max}}\limits_{\ell=i,j}\Big \{t^{K_\ell}_\ell
    +T_\ell^{\texttt{nav}}(p^{K_\ell}_\ell,\,p)\Big \}\right\},
 \end{equation}
 of which the associated time is denoted by~$t^\star_{ij}$ (Line~\ref{alg_line:gen_comm});
and $T_\ell^{\texttt{nav}}(\cdot)$ is defined in~\eqref{eq:meet-constraint}.
Then, if either robot~$i$ or~$j$ departs from~$p^\star_{ij}$ at time~$t^\star_{ij}$
and {immediately} returns to the operator,
the latency constraint can be checked by~\eqref{eq:meet-constraint}.
If it does not hold for either robot~$i$ or~$j$,
it indicates that {either} one needs to return to the operator {before} they meet (Line~\ref{alg_line:need_return}).
In this case, the robot that minimizes the largest latency is selected to return;
and if both robots can't reduce largest latency, both will have to return.
For now, assumed that
robot~$j$ returns to the operator after its latest event~$c^{K_j}_j$.
Since this return event is a meeting event at~$p_{\texttt{h}}$,
it can be appended to~$\widehat{\Gamma}^-_j$ as an additional event (Line~\ref{alg_line:add_return}),
i.e., the new~$C^{K_j}_j$ now denotes the return event.
Consequently, $\Omega^\texttt{h}_{ij}$ is updated by:
\begin{equation}\label{eq:update-omega-h-ij}
  \Omega^\texttt{h}_{ij}[n]=\left \{
    \begin{aligned}
      \mathop{\mathbf{max}}\Big\{\Omega^\texttt{h}_{ij}[n], \Omega^\texttt{l}_{j}[n]\Big\},&  &n\neq j;\\
      t^{K_j}_j+T_i^{\texttt{nav}}(p^{K_j}_j,p_{\texttt{h}}), & &n=j,
    \end{aligned}
  \right.
\end{equation}
which means that the timestamps of the data that
the operator possesses will be updated when robot~$j$ returns (Line~\ref*{alg_line:update_w}-\ref*{alg_line:update_lim}).
On the other hand, if neither robot~$i$ or $j$ should return before they meet,
their communication event is optimized in the following.

\textbf{Communication event}.
To begin with, a traveling salesman problem (TSP) is formulated
by setting~$p^{K_i}_i$ as the starting depot, $p^{K_j}_j$ as the ending location,
and frontiers in~$\widehat{\mathcal{F}}_{ij}$ as the waypoints to visit (Line~\ref{alg_line:tsp}).
Its solution is denoted by~$\widehat{\mathbf{F}}_{ij}$ as a sequence of frontiers.
Then, the next communication event is optimized by an iterative algorithm
as follows.
In each iteration,
a collision-free shortest path~$\widehat{\mathbf{p}}_{ij}$ is generated
that starts from~$p^{K_i}_i$,
traverses the frontiers in~$\widehat{\mathbf{F}}_{ij}$ in order,
and ends at~$p^{K_j}_j$ (Line~\ref{alg_line:gen_path_tsp}).
Within~$\widehat{\mathbf{p}}_{ij}$,
a potential meeting point $\widehat{p}_{ij}$ and time $\widehat{t}_{ij}$ is
determined by~\eqref{eq:select-p} (Line~\ref{alg_line:gen_comm_tsp}).
Consider the following two cases:
(i) If the result satisfies the constraint~\eqref{eq:meet-constraint},
the meeting event~$\widehat{c}_{ij}$ is set to~$(\widehat{p}_{ij},\widehat{t}_{ij})$
and the iteration stops;
(ii) However, if none of the waypoints in~$\widehat{\mathbf{p}}_{ij}$
can satisfy the constraint~\eqref{eq:meet-constraint} (Line~\ref{alg_line:continue_itr}),
the frontier in~$\widehat{\mathcal{F}}_{ij}$ with the maximum cost
is removed from~$\widehat{\mathcal{F}}_{ij}$ (Line~\ref{alg_line:select_frontier}-\ref{alg_line:remove_frontier}).
Then, the iteration continues.
Given a frontier~$f\in \widehat{\mathcal{F}}_{ij}$,
its cost is calculated in two cases:
First, if there is {no} prioritized region given by the operator,
it is defined by:
\begin{equation}\label{eq:frontier-cost}
cost(f)\triangleq T_i^{\texttt{nav}}(f,\, p_{\texttt{h}})
+\textbf{max}_{n\in \{i,j\}}\big\{T_n^{\texttt{nav}}(f,\, p^{K_n}_n))\big\},
\end{equation}
as the estimated time for robots~$i$ or~$j$ to reach~$f$
and then return to the operator;
Second, if there is a prioritized region by~$Q_1$ request,
the cost is given by~$cost(f)\triangleq T_i^{\texttt{nav}}(f,\, p_c)$,
where~$p_c$ is the center of the prioritized region,
as shown in Fig.~\ref{fig:illu-priority}.
As proven in Lemma~\ref{prop:exist-solution}, the above procedure terminates
in finite steps.

Thus, the respective paths~$\mathbf{p}^{K_i+1}_i$ and~$\mathbf{p}^{K_j+1}_j$
before~$\widehat{c}_{ij}$ in~$\widehat{\Gamma}^+_i$ and~$\widehat{\Gamma}^+_j$
are determined by splitting~$\widehat{\mathbf{p}}_{ij}$ into two segments at~$\widehat{p}_{ij}$
(Line~\ref{alg_line:split}-\ref{alg_line:return}).
Lastly, once the final plans~$\widehat{\Gamma}_i$ and~$\widehat{\Gamma}_j$ are
determined, data~$D_i$ in~\eqref{eq:exchanged-data} is updated with the
revised~$\Gamma_i$ and~$\mathcal{F}_i$ (the same for robot~$j$).
More importantly, the timestamps in~$\Omega_i$ are updated given this
new meeting event, i.e.,
$\Omega^\texttt{h}_i[n]=\Omega^\texttt{h}_{ij}[n]$,
where $\Omega^\texttt{h}_{ij}[n]$ is updated in~\eqref{eq:update-omega-h-ij}; and
\begin{equation}\label{eq:update-omega^l_i}
  \Omega^\texttt{l}_i[n]=\left \{
  \begin{aligned}
    &\widehat{t}_{ij},\, n \in \{i,j\};\\
    &\mathop{\mathbf{max}}\Big\{\Omega^\texttt{l}_i[n],\Omega^\texttt{l}_j[n]\Big\},\, n\neq i,j;
  \end{aligned}
  \right.
\end{equation}
where $\widehat{t}_{ij}$ is the confirmed meeting time.
Note that $\Omega_j$ is updated in the same way by robot~$j$.

\begin{figure}[t]
	\centering
	\includegraphics[width=0.97\linewidth]{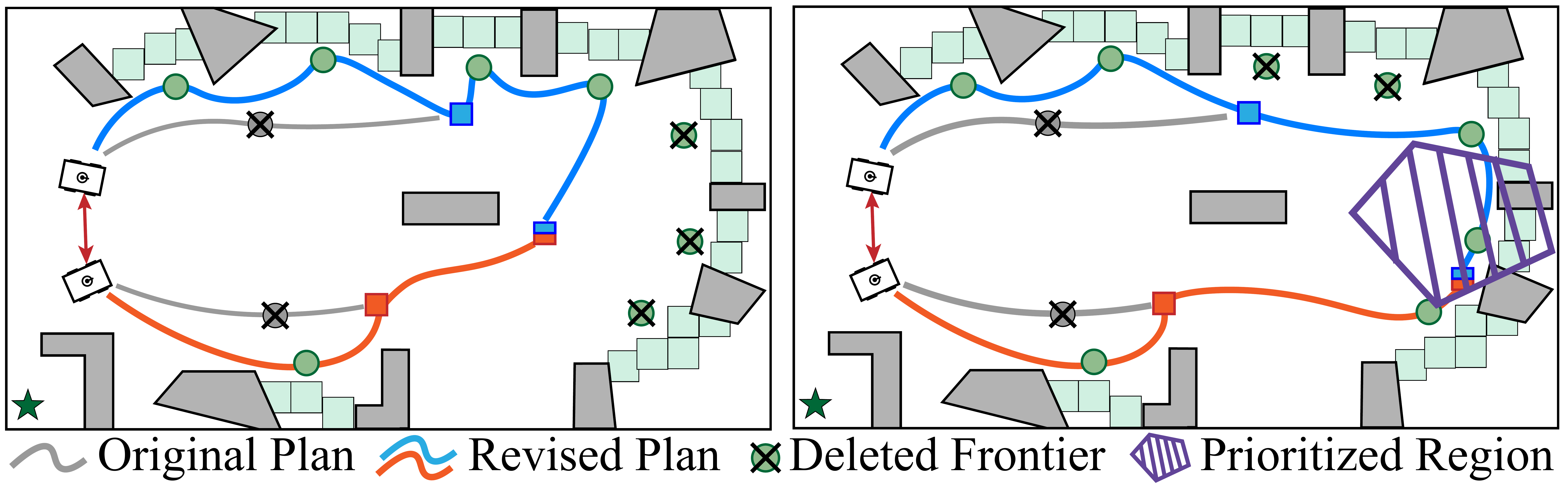}
	\caption{Illustration of the revised plans with (right) and without (left) the prioritized region.}
	\label{fig:illu-priority}
	\vspace{-7mm}
\end{figure}

\begin{remark}\label{remark:opt}
  The derived $\widehat{\Gamma}^+_i$ and~$\widehat{\Gamma}^+_j$ are optimal
  given the information available to robots~$i,j$ in this round, i.e.,
  the number of frontiers in~$\widehat{\Gamma}^+_i$ and~$\widehat{\Gamma}^+_j$
  is maximized subject to the constraints on the next communication event in~\eqref{eq:meet-constraint}.
  \hfill $\blacksquare$
\end{remark}


\subsection{Dynamic Movement of Operator}
\label{subsec:dynamic}
The $Q_2$ request in~\eqref{eq:problem} specifies that the operator wants
to move to another location within the explore area.
Indeed, if the operator remains static,
due to the latency constraint by~$T_\texttt{h}$ and the scarce communication,
the area that a robotic fleet can explore by \emph{any strategy} is strictly bounded.
i.e., the boundary of explored area to the operator
is upper-bounded by:
\begin{equation}\label{eq:max_area}
d_{\texttt{max}}\leq (1-\frac{1}{2^N})T_{\texttt{h}}v_{\texttt{max}}+Nd_{\texttt{com}},
\end{equation}
where~$v_{\texttt{max}}$ is the maximum speed of all robots;
$N$ is the total number of robots;
and $d_{\texttt{comm}}$ is the maximum range of communication.
An illustration of how the explored area changes
w.r.t. the number of robots is shown in Fig.~\ref{fig:human_move_fig}.
It can be seen that the exploration efficiency drastically decreases
as the robots are close to above boundary.
Consequently, if the operator remains static, it often happens that
the deployed fleet can not explore the workspace fully.

Therefore,
the operator is allowed to move within the explored area.
However, the operator can not move arbitrarily,
due to the same latency constraints in~\eqref{eq:delta-r}.
More specifically, each time robot~$i^{\star}\in \mathcal{N}$ returns,
the operator can choose a new pose~$p'_{\texttt{h}}$
within the latest map,
under the following constraint:
\begin{equation} \label{eq:base-location}
   T_i^{\texttt{nav}}(p'_{\texttt{h}},\, p^{k_i}_i)\leq
   T_i^{\texttt{nav}}(p_{\texttt{h}},\, p^{k_i}_i),
  \; \forall k_i\leq K_i,\, \forall i\in \mathcal{N};
\end{equation}
where~$p_{\texttt{h}}$ is the current pose of the operator;
$K_i$ is the number of meeting events for robot~$i$;
$p^{k_i}_i$ is the pose of the $k_i$-th communication event
for robot~$i$ out of its~$K_i$ events.
It requires that the new pose~$p'_{\texttt{h}}$ should be no further
to the confirmed meeting events
(in terms of navigation time for robot~$i$)
than the current pose~$p_{\texttt{h}}$.
In this way, the condition~\eqref{eq:meet-constraint} is still fulfilled
if the operator moves to~$p'_{\texttt{h}}$.
\begin{definition}[Feasible Region]\label{def:feasible}
  The set of all new poses of the operator~$p'_{\texttt{h}}$ that
  satisfy~\eqref{eq:base-location} is called the \emph{feasible region}
  and denoted by~$\mathcal{P}_{\texttt{h}}$. \hfill $\blacksquare$
\end{definition}

In other words, if the operator sets~$p^{\star}_{\texttt{h}}$
as the~$Q_2$ request that is outside of the feasible region,
the path to~$p^{\star}_{\texttt{h}}$ is decomposed into several segments
via intermediate waypoints such that the subsequent waypoint is always
feasible once the previous one is reached.
The operator can choose among these waypoints and move gradually
the desired~$p^{\star}_{\texttt{h}}$. 
For instance,
once the new pose~$p'_{\texttt{h}}$ is chosen,
it is sent to any return robot,
and then the operator moves to~$p'_{\texttt{h}}$.
This updated operator pose is broadcast to other robots via the
subsequent communication events among the robots,
such that the robots can plan the next return events
based on the updated pose.
This process is repeated until the operator reaches~$p^{\star}_{\texttt{h}}$,
as shown in Fig.~\ref{fig:human_move_fig}.
\begin{figure}[t]
  \centering
  \begin{subfigure}[b]{0.47\textwidth}
      \includegraphics[width=\textwidth]{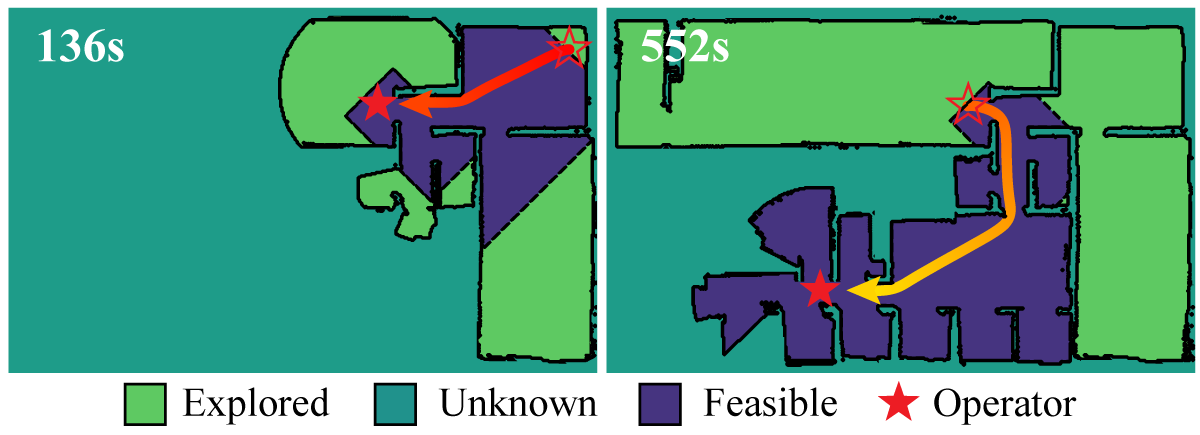}
  \end{subfigure}
  \begin{subfigure}[b]{0.47\textwidth}
      \includegraphics[width=\textwidth]{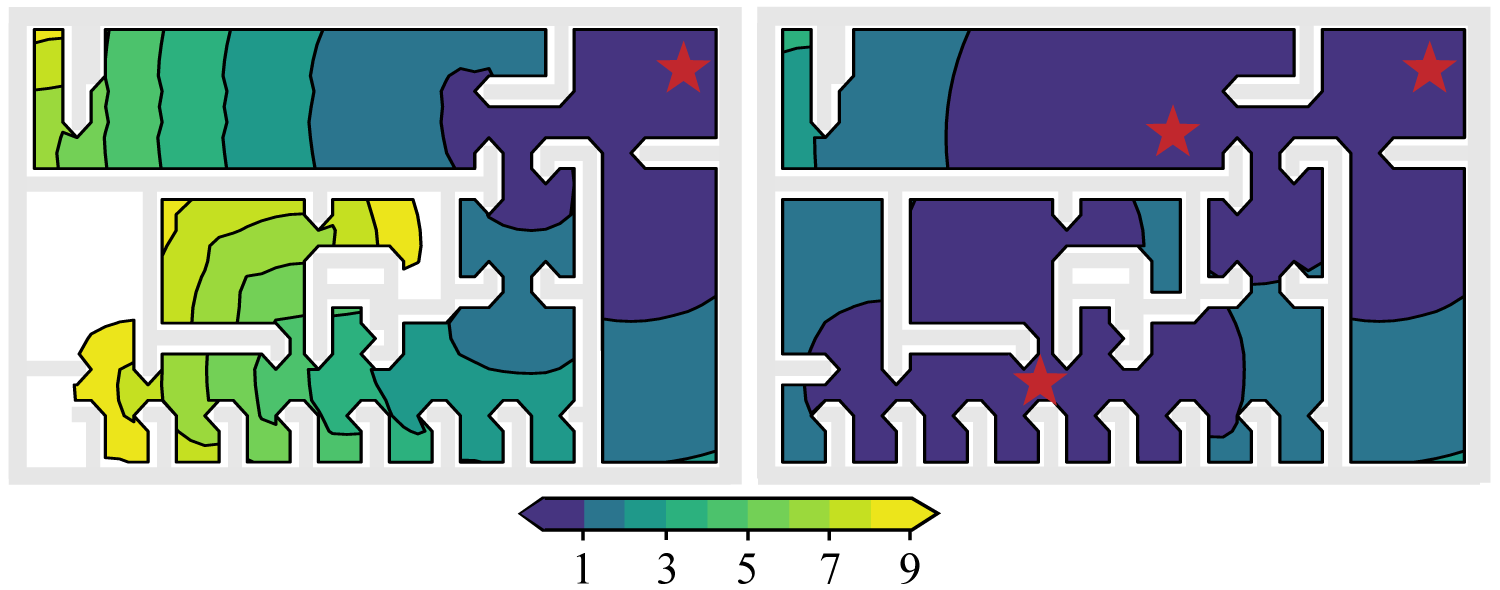}
  \end{subfigure}
  \caption{
    \textbf{Top}: Illustration of the feasible region for the operator
    defined in~\eqref{eq:base-location} at different time instants.
    \textbf{Bottom}: Comparison of the maximum explored area
    when the operator is static (\textbf{Left}) or dynamic (\textbf{Right}),
    where the color bar indicates the number of robots required.
    }\label{fig:human_move_fig}
    \vspace{-4mm}
  \end{figure}

  \begin{remark}\label{remark:allowed}
    Note that the operator behavior of choosing the next new pose can
    be sub-optimal (different from the planned~$p'_{\texttt{h}}$)
    or even arbitrary, as long as it belongs to the feasible
    region~$\mathcal{P}_{\texttt{h}}$ from Def.~\ref{def:feasible}.
    \hfill $\blacksquare$
  \end{remark}


\subsection{Discussion}\label{subsec:discussion}

\subsubsection{Correctness}
The existence of a solution by Alg.~\ref{alg:opt-com} and the guarantee on the~$Q_0$ request
are stated below, of which the proofs are given in the Appendix of Sec.~\ref{sec:app}.

\begin{lemma}\label{prop:exist-solution}
  Alg.~\ref{alg:opt-com} always returns a feasible solution
  during each intermittent communication for all rounds.
\end{lemma}
\begin{theorem}\label{theo:Ts-constraint}
  Under the proposed overall strategy,
  the latency constraint by~$T_{\texttt{h}}$ in~\eqref{eq:frequency} is fulfilled.
\end{theorem}
\usetikzlibrary{shapes.geometric}
\usetikzlibrary{patterns}
\pgfplotstableread[col sep=comma]{
	type, Explore, Meet, Return
	Ring, 48.4, 48.2, 3.4
	Line, 34.9, 50.6, 14.5
	Full, 25.0, 67.7, 7.3
  Tree, 44.7, 44.7, 10.6
  Kite, 31.7, 61.3, 7.0
	}\mytable

\definecolor{blueaccent}{RGB}{0,150,214}
\definecolor{greenaccent}{RGB}{0,139,43}
\definecolor{purpleaccent}{RGB}{130,41,128}
\definecolor{orangeaccent}{RGB}{240,83,50}
\tikzstyle{topo_node}=[circle,draw=black,fill=green!50!black,thick,inner sep=1.2pt,scale=1.6]
\tikzstyle{topo_edge}=[draw=black,thick]

\begin{figure}[t]
	\centering
  \begin{subfigure}[t]{0.40\textwidth}
    \includegraphics[width=\textwidth]{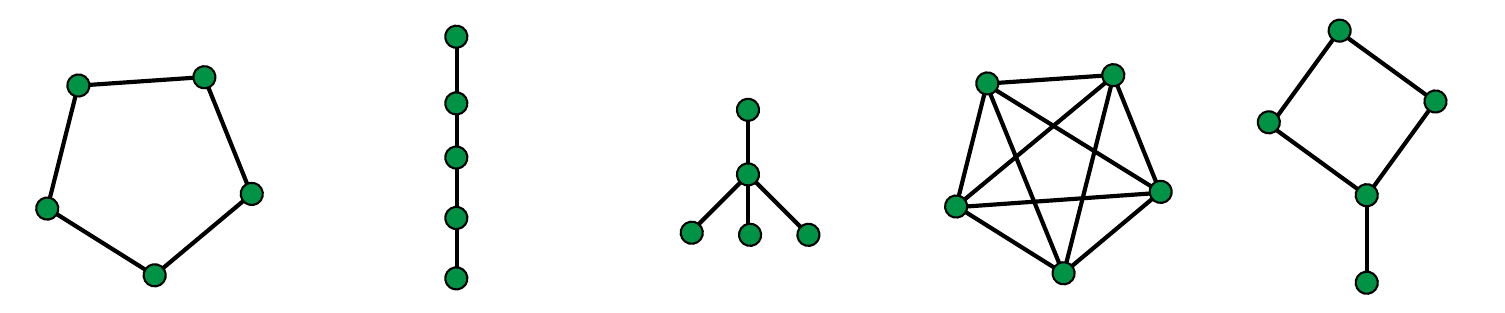}
  \end{subfigure}
  \begin{subfigure}[t]{0.40\textwidth}
    \begin{tikzpicture}
      \begin{axis}[
        width=\linewidth,
        ybar,
        bar width=5.5pt,
        ymin=0,
        ymax=100,
        width=1.0\linewidth,
        height=0.6\linewidth,
        ymajorgrids=true,
        grid style=dashed,
        enlarge x limits={abs=25pt},
        legend style={at={(0.38,0.95)},
        anchor=north,legend columns=-1, font=\small},
        ylabel={Time Percentage (\%)},
        y label style={at={(axis description cs:0.10,0.5)},anchor=south, font=\small},
        symbolic x coords={Ring,Line,Tree,Full,Kite},
        x label style={font=\tiny},
        xtick=data,
        cycle list={blueaccent,greenaccent,purpleaccent,orangeaccent}
      ]
        \pgfplotsinvokeforeach{Explore, Meet, Return}{
        \addplot+[draw=black,fill,] table[x=type,y=#1]{\mytable};
          \addlegendentry{#1}
        }
    \end{axis}
    \end{tikzpicture}
  \end{subfigure}
  \vspace{-2mm}
  \caption{
    Comparison of exploration efficiency (\textbf{Bottom})
    under different communication topologies (\textbf{Top}),
    for the small office in Sec.~\ref{sec:experiments}
    with $5$ robots and $T_\texttt{h}=150s$.
  }
  \label{fig:comm-topology}
  \vspace{-6mm}
\end{figure}
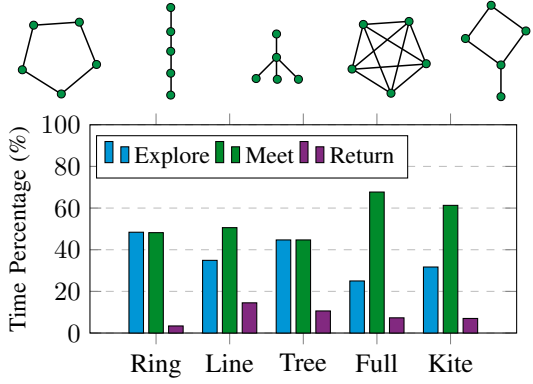

\subsubsection{Choice of Communication Topology}\label{subsec:topology}
Although the proposed strategy above is valid for any connected
communication topology~$\mathcal{G}$,
a fixed ring topology among the robots is adopted to ensure {a high}
efficiency of information propagation,
while a dynamic connection between the robots and the operator is allowed.
More specifically,
the resulting schedule of communication~$C_0$ in~\eqref{eq:round}
satisfies that~$n'_i=n_{i+1}$, $\forall n_i\in C_0$,
i.e., robot~$n'_i$ is simply the succeeding robot of~$n_i$,
and~$n_i$ is the preceding robot of~$n'_i$.
Whenever a return event is required by~\eqref{eq:select-p}
during the meeting event of two robots,
the preceding robot always returns to the base.
It can be shown that under this topology,
the data from any robot can be transmitted to any other robot
after maximum~$N$ meeting events.
More importantly,
this topology ensures that when two robots meet,
they would have obtained the information from \emph{all} previous meeting events.
This feature yields the ring topology particularly suitable for
the scenarios of human-robot interaction,
where the operator can be informed about the overall progress
of the whole fleet when a single robot returns.
Second,
this topology is essential for the~${Q}_2$ request,
due to the following two aspects:
(I) The calculation of feasible region~$\mathcal{P}_{\texttt{h}}$
in Def.~\ref{def:feasible} relies on the knowledge of
meeting events of {all} robots,
which can be easily obtained under the ring topology;
(II) After a~${Q}_2$ request is sent and the operator moves away,
the ring topology ensures that any other robot that is scheduled
to return can be informed about the new operator position,
before it actually returns.
Lastly, as shown in Fig.~\ref{fig:comm-topology}
and later in Sec.~\ref{sec:experiments},
the ring topology often leads to a higher rate of time spent on exploration
and lower rate on returning to the operator, compared
with other topologies.

\begin{figure}[t]
  \centering
  \includegraphics[width=0.95\linewidth]{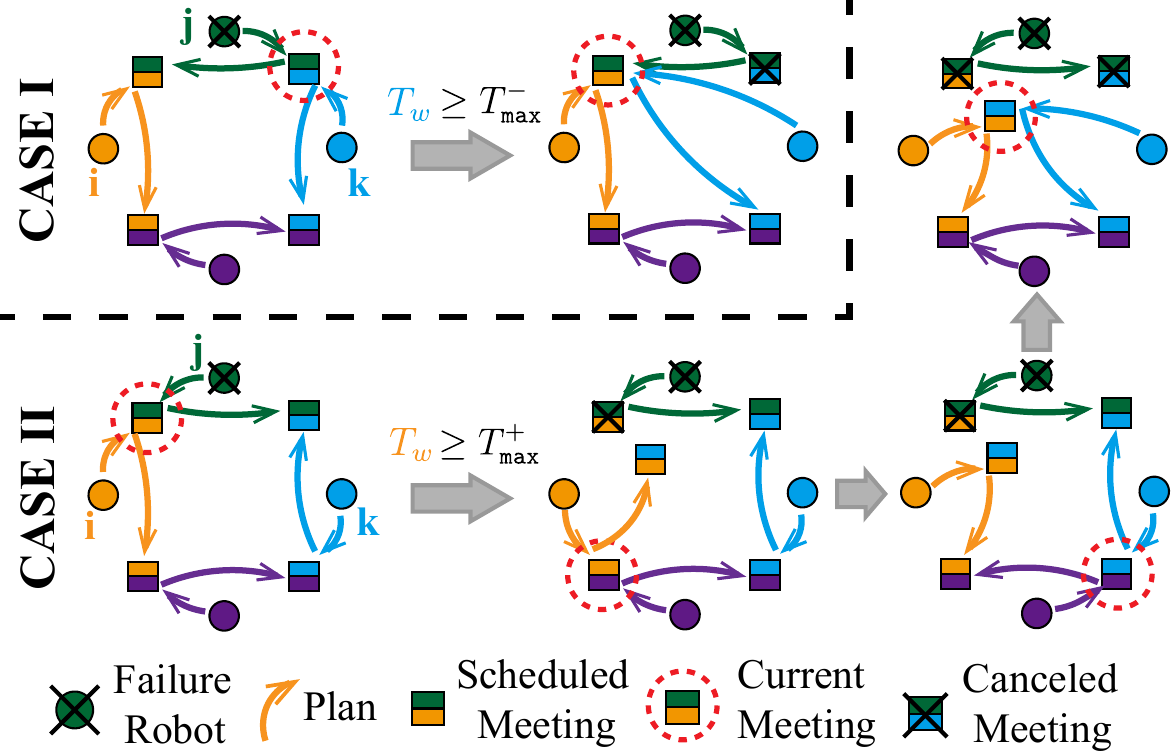}
  \vspace{-2mm}
  \caption{Illustration of two cases during
      the failure detection and recovery mechanism
      in Sec.~\ref{subsubsec:failure-recovery}.}
  \label{fig:failure-cases}
  \vspace{-3mm}
  \end{figure}

\subsubsection{Local Plan Adaptation}\label{subsubsec:local-plan}
During the online execution,
it is possible that a frontier that is assigned to one robot
is explored earlier by another robot on its way to other frontiers.
In addition, it might have generated additional frontiers
and there is sufficient time left before the next meeting event.
More importantly,
the preliminary plans~$\Gamma_i,\, \Gamma_j$ derived via Alg.~\ref{alg:opt-com}
may exceed the planned meeting time due to motion uncertainties,
which may lead to excessive waiting time and a potential violation of the~$Q_0$ request.
Therefore, it is important for robots to update their local plans as follows.
Each time when robot~$i\in \mathcal{N}$ reaches a frontier point,
it extracts the set of frontiers~$\mathcal{F}_t$
within its latest map,
among which the condition~$T_i^{\texttt{nav}}(p_i, f)
+T_i^{\texttt{nav}}(f, p_c)<t_c-t$
is checked for each~$f\in \mathcal{F}_t$,
with~$p_i$ being the current robot pose,
$p_c$ and~$t_c$ the pose and time of the next communication event.
If so, it ensures that robot~$i$ can reach $p_c$ before $t_c$ if it visits the frontier~$f$.
Then, the priority of remaining frontiers is measured by
$\chi(f)\triangleq
\omega_1 {\mathbf{mean}}_{f^-\in \mathcal{F}^-_i}
\{T_i^{\texttt{nav}}(f, f^-) \} -\omega_2
T_i^{\texttt{nav}}(p, f) -  \omega_3 {\mathbf{mean}}_{f^+\in \mathcal{F}_i}
\{T_i^{\texttt{nav}}(f, f^+)\}$,
where
$\mathcal{F}^-_i$ are frontiers allocated to other robots;
$\mathcal{F}_i$ are frontiers allocated to robot~$i$;
and~$\omega_1,\omega_2, \omega_3\geq 0$ are parameters.
Thus, frontiers are prioritized if they are close to the current robot pose
and the frontiers already assigned to robot~$i$,
and far from frontiers assigned to other robots.
The frontier with the highest priority is chosen as the next goal point.
If no frontiers satisfy this condition,
robot~$i$ moves to the next meeting event directly.

\subsubsection{Spontaneous Meeting Events}
As the navigation path is determined by the SLAM module in~\eqref{eq:slam},
it often occurs that robot~$i$ meets with another robot~$k$
on its way to the meeting event with robot~$j$,
which is called a \emph{spontaneous} meeting event.
In this case,
they exchange their local data~$D_i$ and~$D_k$ including merging the local maps.
However, they do not coordinate the next meeting event nor modify their local plans,
such that the communication topology~$\mathcal{G}$ can be maintained,
instead of growing into a full graph.

\subsubsection{Failure Detection and Recovery}\label{subsubsec:failure-recovery}
Robots can often fail when operating in extreme environments as considered in this work.
Thus, a failure detection and recovery mechanism is essential for the safety of the
functional robots.
The key is to ensure that the communication topology
can be dynamically adjusted to form a smaller ring after one robot fails,
still via only intermittent local communication.
Assume that robot~$j$ fails at time~$t_f>0$,
and its proceeding and succeeding neighbors in~$\mathcal{G}$
are robot~$i$ and robot~$k$, respectively.
As illustrated in Fig.~\ref{fig:failure-cases}, the mechanism can
be sorted into two cases:
(I) If robot~$j$ fails on its way to meet robot~$i$,
robot~$i$ will wait for robot~$j$ at the communication
location~$c_{ji}$ agreed from the previous round, for a waiting period of $T^-_{\texttt{max}}$.
Then, robot~$i$ assumes that robot~$j$ has failed and \emph{directly} moves to
the meeting location~$c_{jk}$ of robots~$j$ and~$k$,
which has been propagated to robot~$i$ via previous communication events.
During communication, robots~$i$ and $k$ plans for their next meeting event
as described previously;
(II) If robot~$j$ fails \emph{before} meeting with robot~$k$,
robot~$k$ detects the failure of robot~$j$ in the same way as robot~$i$
does in the first case, with a waiting time~$T^+_{\texttt{max}}$.
However, robot~$k$ in this case cannot directly meet with robot~$i$,
as robot~$i$ has already communicated with robot~$j$ and most likely performs exploration.
Moreover, the meeting event between robot~$i$ and its predecessor is unknown to robot~$k$
due to the failure of robot~$j$.
Therefore, robot~$k$ first determines a meeting location $p_{ki}$ locally,
and then communicates with its successor to propagate this information.
After several communication events, this location is propagated to robot~$i$
along with the message that robot~$j$ has failed.
Thus, robot~$i$ cancels the meeting event with robot~$j$ and heads to $p_{ki}$
to meet with robot~$k$ instead.
In both cases,  the communication topology~$\mathcal{G}$ is adjusted to form a smaller ring,
by removing the failed robot.
It is worth noting that $T^+_{\texttt{max}}$ should be larger than $T^-_{\texttt{max}}$,
such that robot~$i$ can meet with robot~$k$ in case (I),
before robot~$k$ judges that robot~$j$ has failed and moves away.

\subsubsection{Time-varying Latency Constraint}\label{subsubsec:dynamic-Th}
During online execution,
the operator may require different frequency of status updates
according to the current progress or other factors.
Therefore, the proposed framework allows the operator to
dynamically adjust the constraint~$T_{\texttt{h}}$ as the~$Q_0$ request,
i.e., by specifying a new threshold~$T'_{\texttt{h}}$ for the future rounds
when a robot returns.
Thus, this new threshold is propagated to the fleet and is reflected
automatically in the subsequent intermittent communication events.


\begin{figure}[t]
  \centering
  \begin{subfigure}[b]{0.15\textwidth}
      \includegraphics[width=\textwidth]{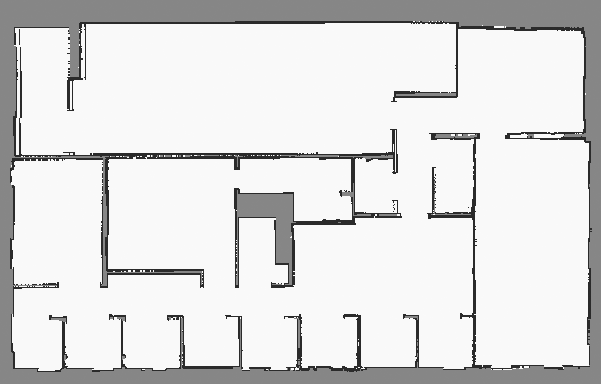}
      \label{fig:office_env}
  \end{subfigure}
  \begin{subfigure}[b]{0.15\textwidth}
      \includegraphics[width=\textwidth]{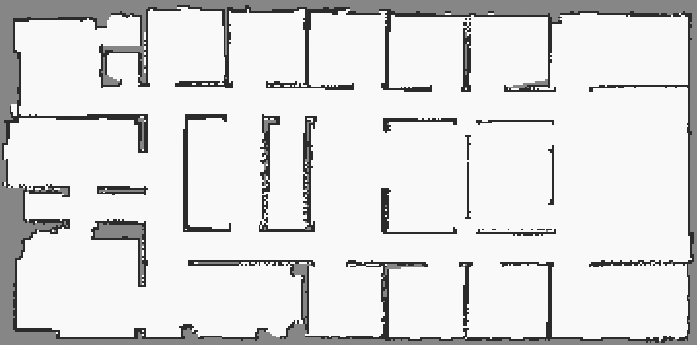}
      \label{fig:forest_env}
  \end{subfigure}
  \begin{subfigure}[b]{0.15\textwidth}
      \includegraphics[width=\textwidth]{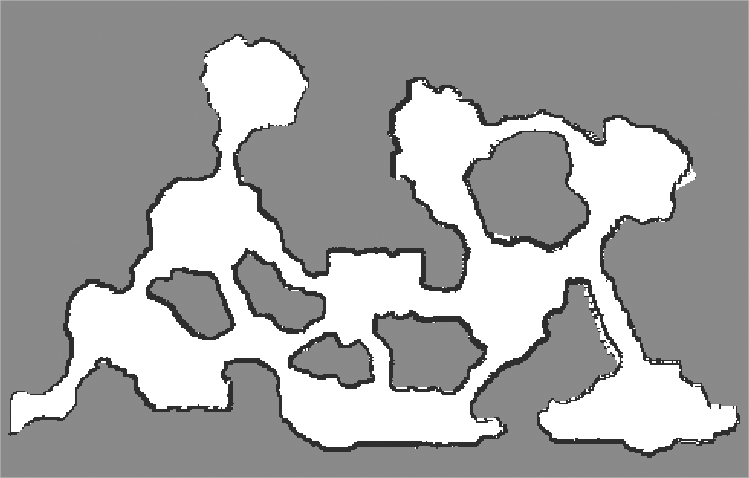}
      \label{fig:cave_env}
  \end{subfigure}
  \vspace{-4mm}
  \caption{Test environments:
  large office with three separated long corridors (\textbf{Left}),
  small office with more connected rooms (\textbf{Middle}),
  and subterranean cave (\textbf{Right}).}
  \label{fig:env}
  \vspace{-3mm}
  \end{figure}

\section{Numerical Experiments} \label{sec:experiments}
To further validate the proposed method,
numerical simulations and hardware experiments
are presented in this section.
The proposed method is implemented in \texttt{Python3}
within the framework of \texttt{ROS},
and tested on a laptop with an Intel Core i7-1280P CPU.
The operator interacts with the robotic fleet via an Android tablet.
Simulation and experiment videos can be found in the supplementary material.

\begin{figure}[t]
  \centering
  \includegraphics[width=0.9\linewidth]{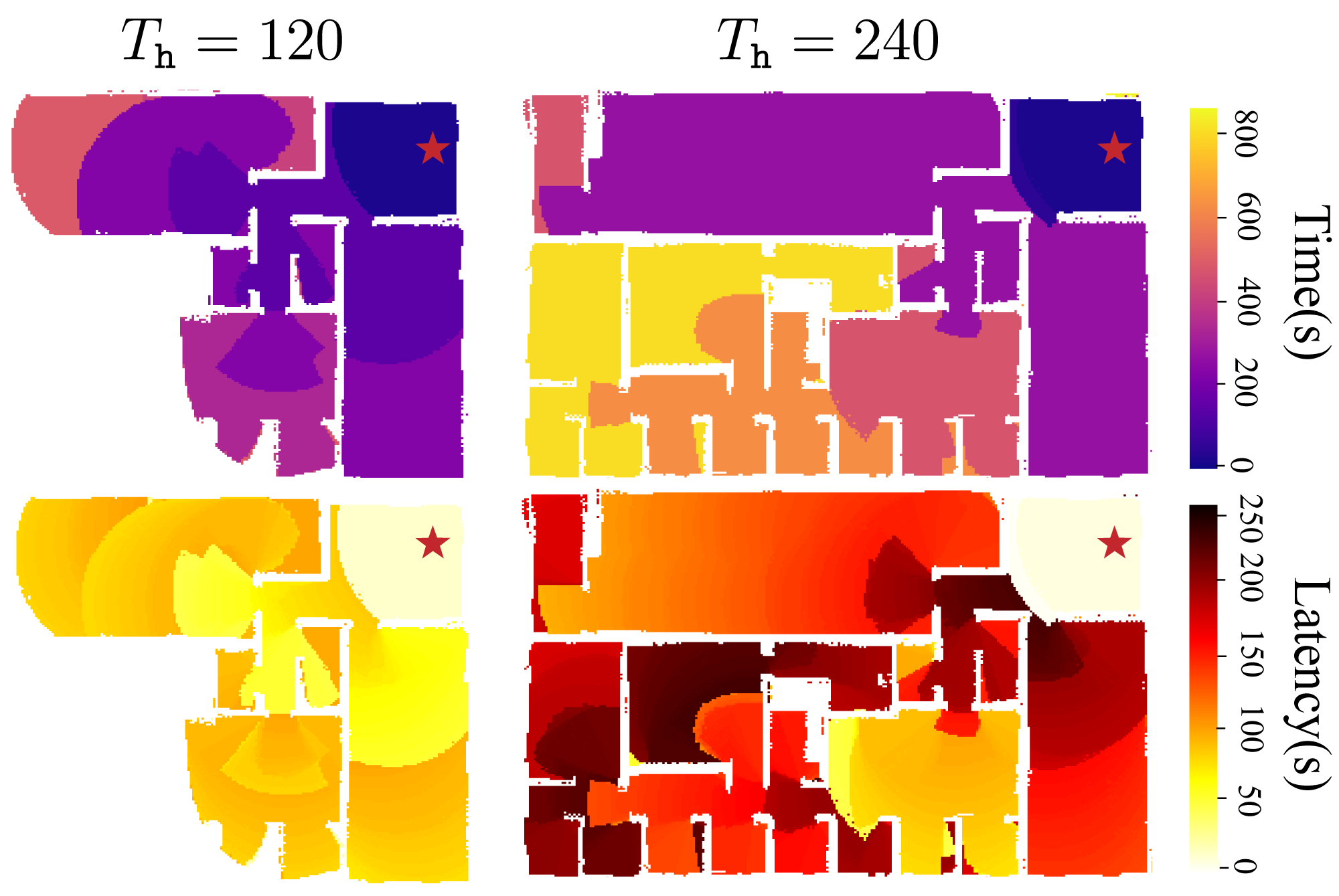}

  \includegraphics[width=0.9\linewidth]{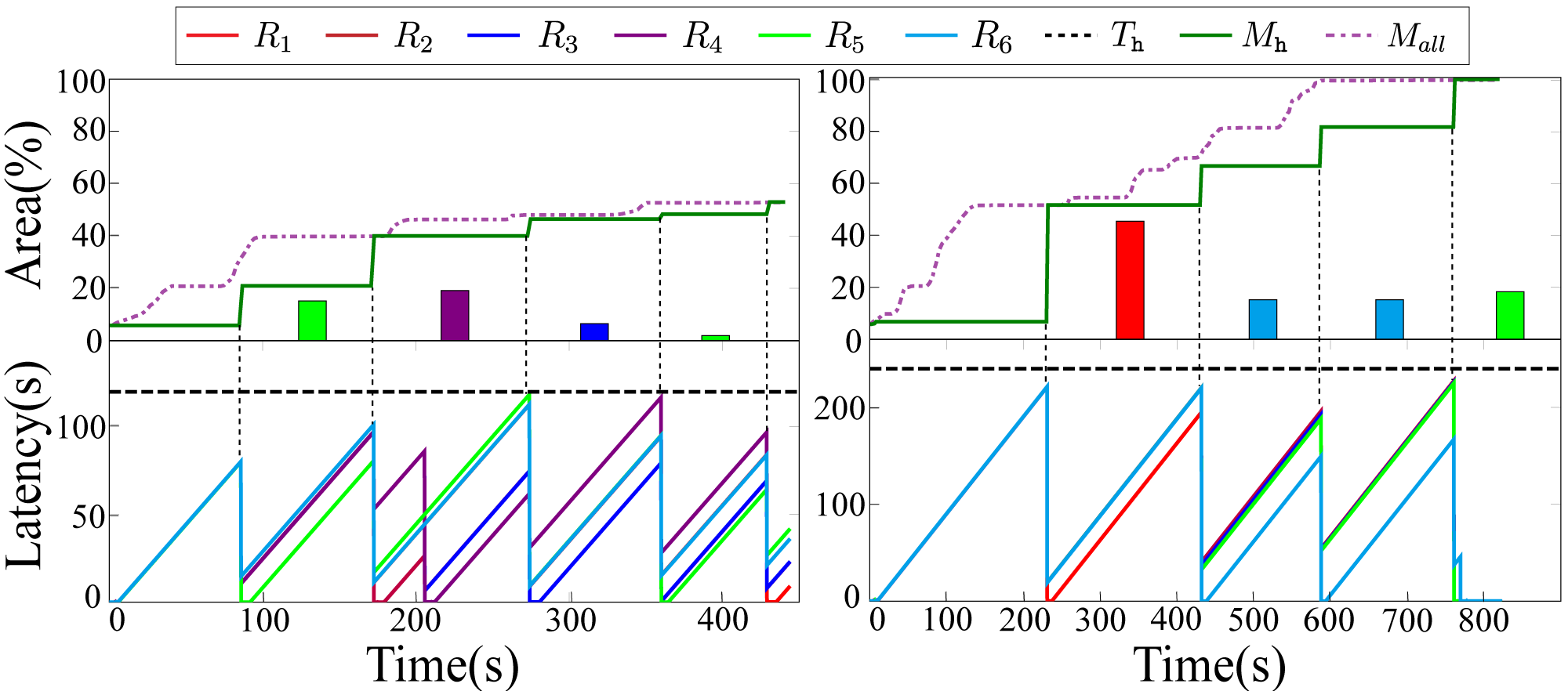}
  \caption{Simulation results of $6$ robots exploring
  a $50m\times30m$ office,
  when the latency constraint~$T_{\texttt{h}}$ is set
  to $120s$ (\textbf{Left}) and $240s$ (\textbf{Right}).
  \textbf{Top}: the final map at the operator (marked by the red star),
  color-filled by the time of being received;
  \textbf{Middle}: the latency of different parts,
  measured by the time difference between being explored
  and being received by the operator;
  \textbf{Bottom}: evolution of the union of
  all explored maps~$M_{\texttt{all}}$,
  the local map~$M_{\texttt{h}}$ of operator,
  the contribution of return robots,
  and the latency of each robot in~\eqref{eq:frequency}.
  }\label{fig:exp_1}
  \vspace{-7mm}
\end{figure}

\subsection{System Description}\label{subsec:description}

The robotic fleet consists of 6 differential-driven UGVs,
which is simulated in the \texttt{Stage} simulator
for three different environments as shown in Fig.~\ref{fig:env}:
(i) a large office of size $50m\times 30m$, with three separated long corridors;
(ii) a smaller but more connected office of size $40m\times 20m$;
(iii) a subterranean cave of size $50m\times 30m$.
A 2D occupancy grid map is generated via the $\texttt{gmapping}$ SLAM package to generate local map,
with a sensor range of $8m$,
which provides a probabilistic
representation of the environment by assigning values to each cell,
i.e., whether the cell has been explored, free or occupied
by obstacles~\cite{moravec1985high}.
Each robot navigates using the ROS navigation stack \texttt{move\_base},
with a maximum linear velocity of $0.5m/s$ and angular velocity of $0.8rad/s$.

Moreover,
two robots can only communicate with each other
if they are within a communication range of $3.5m$
and have a line of sight,
the same between robots and the operator.
The map merge during communication is handled
by the off-the-shelf ROS package $\texttt{multirobot\_map\_merge}$.
Lastly,
the operator interacts with the robotic fleet through the terminal
and the \texttt{Rviz} interface.
Namely,
the operator can send a request by
first specifying its type ($Q_1$ or $Q_2$)
in the terminal;
and then either select the vertices of prioritized region
with the $\texttt{Publish\_Point}$ tool,
or choose the next desired location with the $\texttt{2D\_Nav\_Goal}$ tool.

\subsection{Results}\label{subsec:results}
The results present:
(i) how the proposed method performs under different environments and latency constraints;
(ii) how the human-robot interactions (i.e., $Q_1$ and $Q_2$ requests)
are handled during the exploration process.

\subsubsection{Latency-constrained Exploration}\label{subsubsec:latency}
As shown in Fig.~\ref{fig:exp_1},
the operator is located at the top-right corner of the
the first large office,
which is also where all robots start.
The latency constraint~$T_\texttt{h}$ is set to~$120s$
and~$240s$ for two cases.
It can be seen that the fleet can not
explore the whole environment when $T_\texttt{h}=120s$
with only $53\%$ coverage,
while the whole map is obtained when $T_\texttt{h}=240s$.
Moreover, the latency constraint is satisfied at all time in both cases,
with the maximum latency among all robots being~$118s$ and $227s$, respectively.
In other words,
the explored map is updated to operator in time for both cases,
and a smaller $T_\texttt{h}$ leads to much more frequent updates.
Specifically, when $T_\texttt{h}=120s$,
the return events occur at $87s$, $173s$, $207s$, $275s$, $361s$ and $431s$,
during which the newly-explored map is increased by~$15\%$, $19\%$, $0\%$, $6\%$, $2\%$
and $4\%$, respectively.
Thus, the efficiency quickly decreases as the explored area extends,
i.e., most of the time is spent on navigation and communication,
rather than exploration.
Similar phenomenon can be observed when $T_\texttt{h}=240s$,
where the operator map is increased by~$45\%$, $15\%$, $15\%$, $18\%$
in four return events.
Lastly, it is worth noting that although the whole map is obtained
when $T_\texttt{h}=240s$,
the operator can not obtain a timely update of the fleet status
due to the large latency,
i.e., only at $232s$, $432s$, $588s$ and $762s$.
The left-bottom area is the latest to be explored at $592s$.
This further validates that simply increasing $T_\texttt{h}$
can not guarantee a timely update of the fleet status,
thus movement of operator is essential.
\begin{figure}[t]
    \centering
    \includegraphics[width=0.9\linewidth]{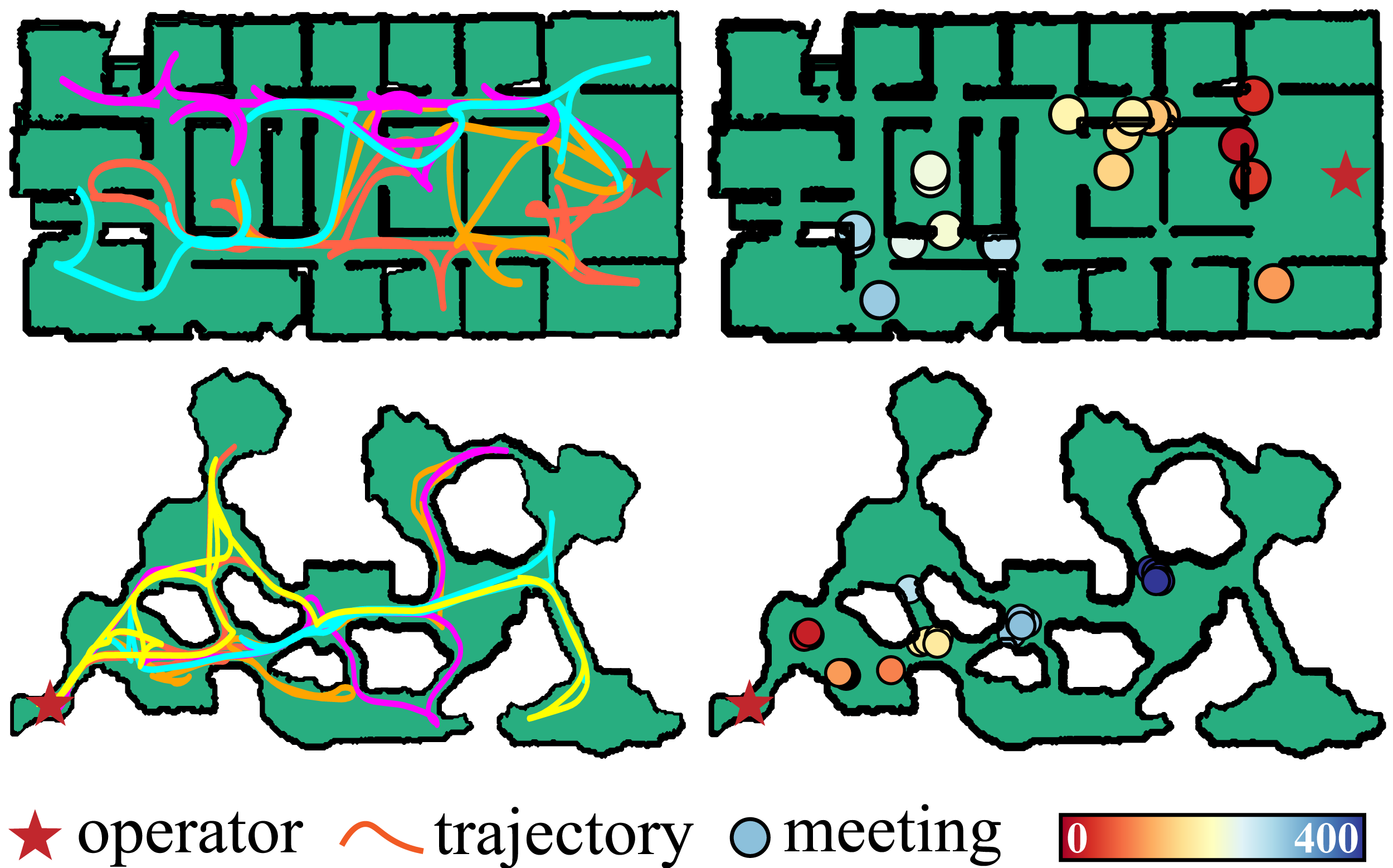}
    \caption{
    Robot trajectories within the final map (\textbf{Left})s
    and the communication points (\textbf{Right})
    in the small office (\textbf{Top}) and the cave (\textbf{Bottom}),
    with $T_\texttt{h}=150s$ and $T_\texttt{h}=200s$.
    Meeting points are color-filled by the
    meeting time (red to blue).}
    \label{fig:office_cave}
    \vspace{-6mm}
    \end{figure}

In addition,
the second small office is tested with $4$ robots and $T_\texttt{h}=150s$,
while the cave environment is tested with $5$ robots and $T_\texttt{h}=200s$.
The final trajectories and communication points are shown in Fig.~\ref{fig:office_cave}.
It can be seen that the whole map is explored for both scenes,
at time $243s$ and $395s$, respectively.
Different from the first large office,
the communication points gradually move forward as the exploration proceeds,
as the robots often navigate a short distance to reach the meeting points,
yielding more time for exploration.
It indicates that the proposed method is particularly effective
for a more connected environment.

\subsubsection{Human-robot Interaction}\label{subsubsec:interaction}
As the operator may interact with the fleet via three
different types of requests,
this section evaluates how the proposed method
handles these interactions during exploration.

\begin{figure}[t]
  \centering
  \includegraphics[width=1.0\linewidth]{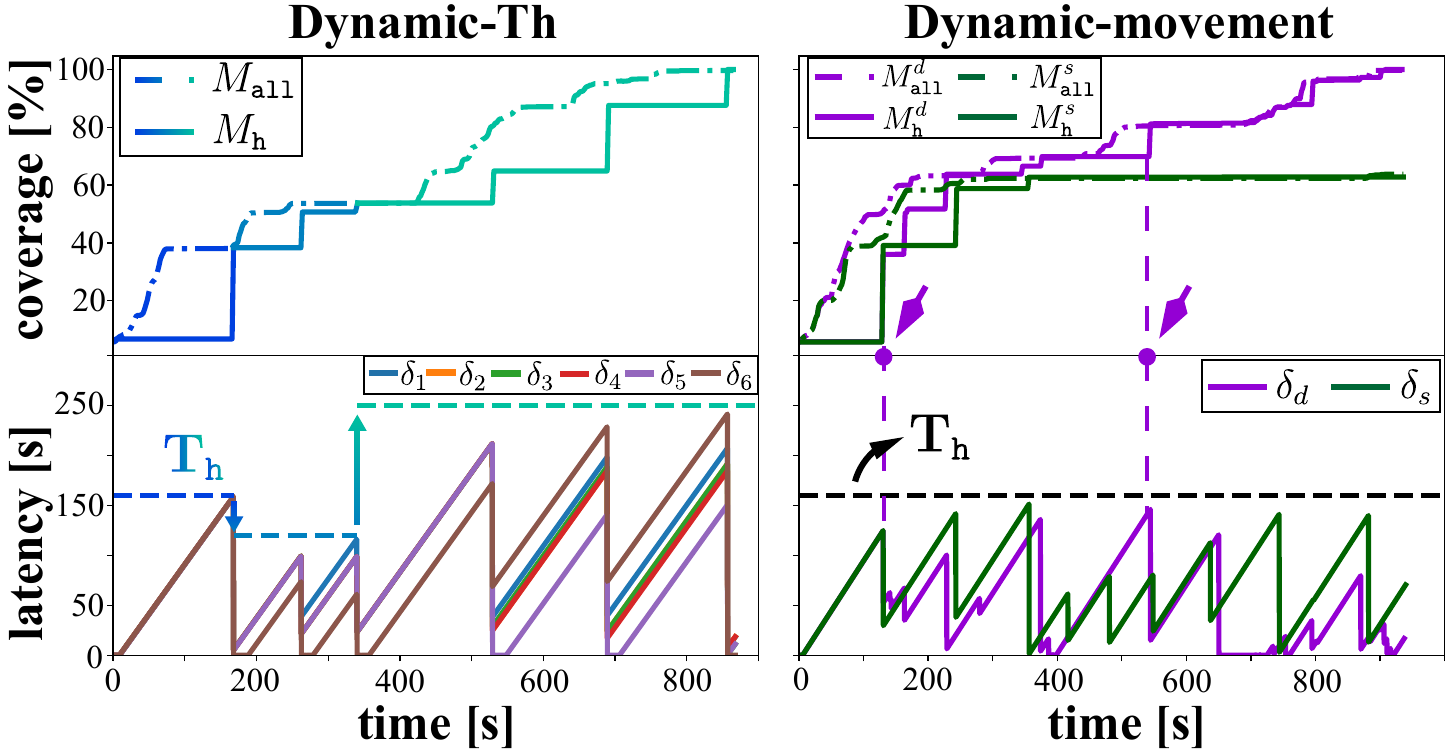}
  \caption{
    \textbf{Left}: Evolution of the explored area $M_\texttt{all}$,
    the local map of the operator~$M_\texttt{h}$,
    and the latency of each robot,
    when $T_\texttt{h}$ changes dynamically;
    \textbf{Right}: Comparison of the exploration efficiency
    when the operator is static (green) or dynamic (purple) (\textbf{Top}),
    and evolution of the largest latency over all robots (\textbf{Bottom}),
    with the arrows indicating when the operator moves.
  }\label{fig:compare_move}
  \vspace{-4mm}
\end{figure}

(I) \textbf{Time-varying latency constraint in~$Q_0$}.
To begin with,
regarding the~$Q_0$ requests where the operator changes $T_\texttt{h}$,
the large office is tested with~$6$ robots and an initial $T_\texttt{h}$ of $160s$.
The results are summarized in Fig.~\ref{fig:compare_move},
which show that $T_\texttt{h}$ is changed twice.
At $t=158s$, the operator decreases $T_\texttt{h}$ to~$120s$,
yielding more frequent updates and smaller exploration efficiency;
while at $t=330s$, $T_\texttt{h}$ is increased to~$250s$,
after which exploration efficiency quickly increases
with less frequent return events.
Lastly, the whole map is obtained at~$860s$,
during which the maximum latency of all robots
never exceeds the bound~$T_\texttt{h}$.

\begin{figure}[t]
  \centering
  \includegraphics[width=0.9\linewidth]{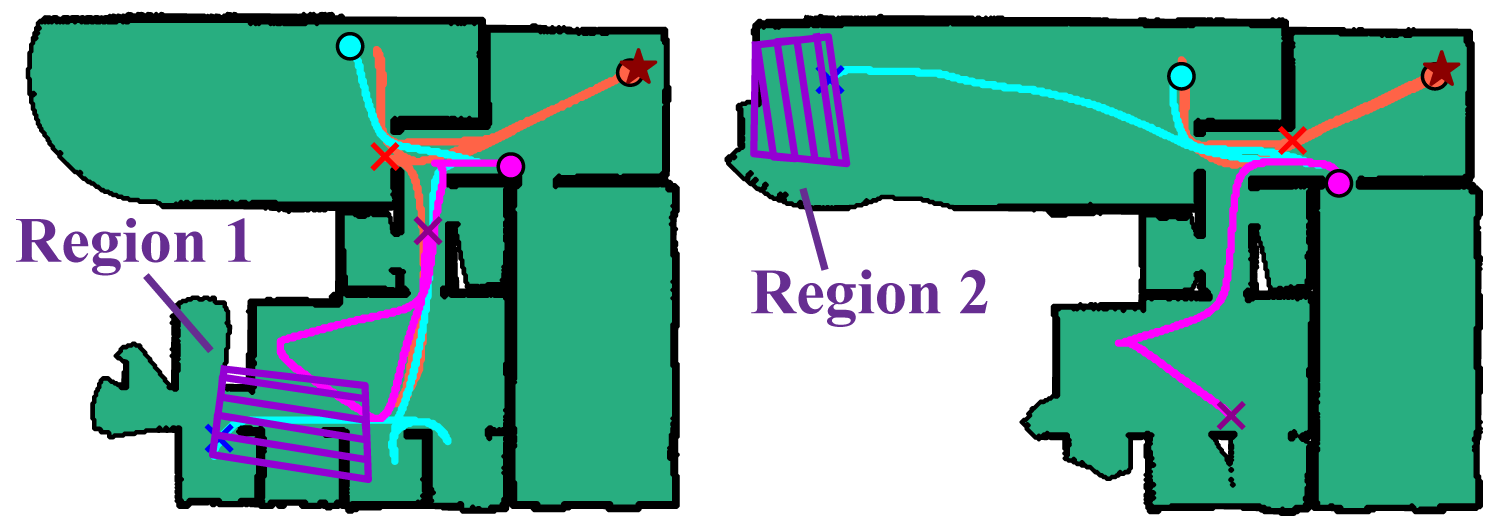}
  \caption{Comparison of the exploration results
  when different operators specify different~$Q_1$ requests
  as prioritized regions (the purple polygon).
  The trajectories (colored lines) start
  from the communication points (filled circles),
  and end when the prioritized region is fully explored (crosses).}
  \label{fig:prioritize}
  \vspace{-3mm}
  \end{figure}

{(II) \textbf{Prioritized region in~$Q_1$}.
Regarding the~$Q_1$ requests where prioritized regions are specified online,
tests are conducted for the large office environment
with $3$ robot number and $T_\texttt{h}=160s$.
A small number of robots is adopted as
this type of requests is useful when there are not enough robots to
explore every branche simultaneously.
Specifically,
the first operator specifies the prioritized region
located at the bottom-left (Region~1) at the first return event;
while the second operator chooses top-left (Region~2),
as shown in Fig.~\ref{fig:prioritize}.
In the first case,
robot $0$ returns at $134.3s$ and receives the request to explore region 1 first,
which is then forwarded to robot $1$ and $2$ at $175.9s$ and $207.4s$.
Then, their updated local plans take into account the prioritized region,
which results in communication points near that region.
Thus, instead of exploring the left branch first,
all the robots first turn to the bottom branch
and explore the prioritized region at $340.8s$.
In contrast for the second case,
the result is known to three robots at $131.1s$, $172.2s$ and $203.0s$ respectively.
Different from the first case,
with the updated plans one robot heads for region~$2$,
one robot returns, and
the other explores the bottom-left branch
as one robot is enough to explore region $2$ at $267.6s$.
Note that in both cases,
it is ensured that there is at least one robot assigned to explore the prioritized region,
even when there are numerous frontiers and only a small number of robots.
Last but not least,
region $2$ is finally explored at $520.8s$ for the first case,
while region $1$ is explored at $501.6s$ for the second case.}

\begin{figure}[t]
  \centering
  \includegraphics[width=\linewidth]{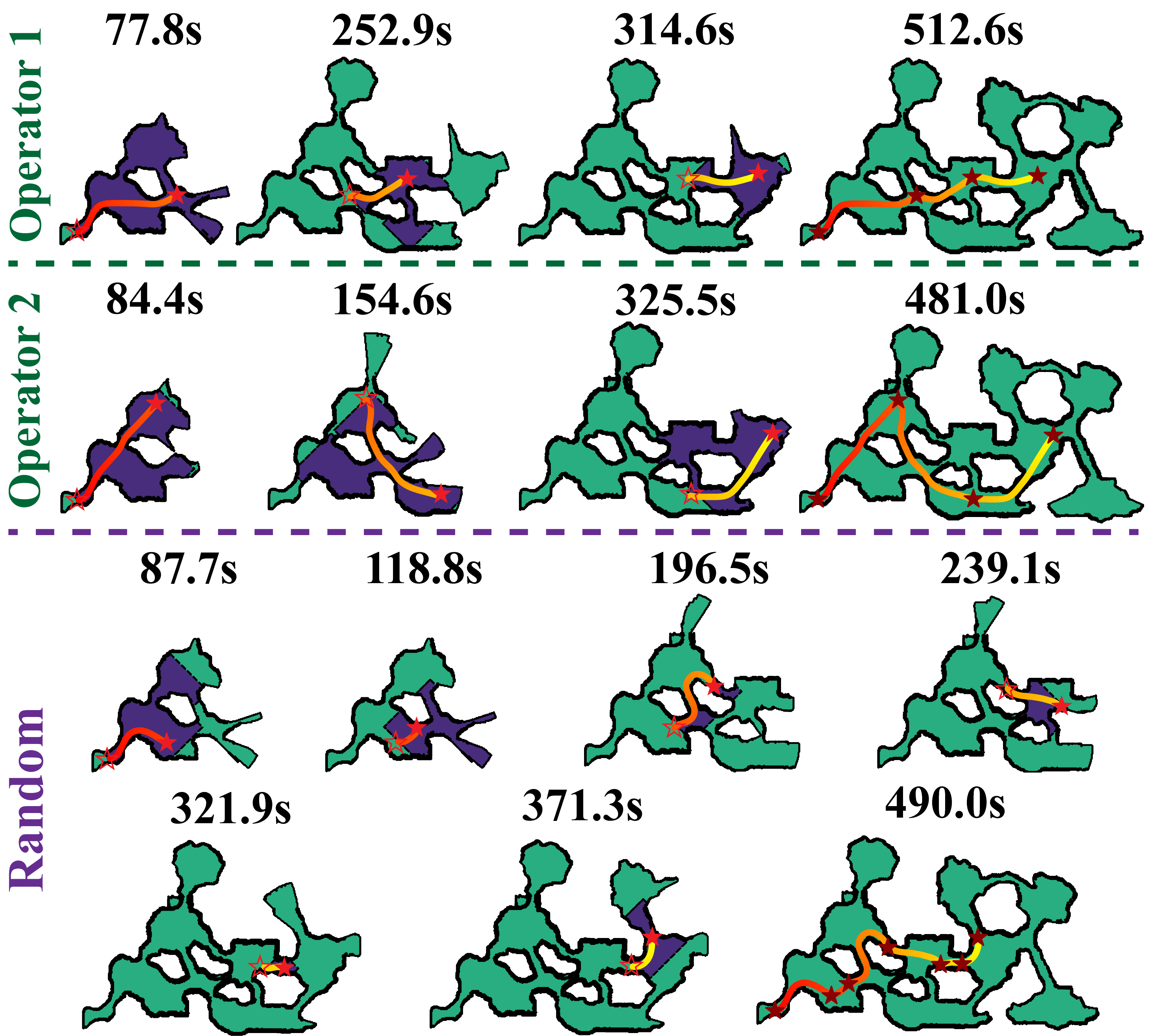}
  \caption{
  Results of how the proposed method reacts to different human behaviors
  of dynamic movement:
  two operators (\textbf{Top}) and a random movement (\textbf{Bottom}).
  Human movements are shown as colored lines with intermediate goal positions (red stars)
  and feasible regions for each movement (purple area). }
  \label{fig:cave_move}
  \vspace{-4mm}
  \end{figure}

(III) \textbf{Dynamic movement of operator in~$Q_2$}.
Regarding the~$Q_2$ requests where the operator
moves to a desired location,
the large office is tested with~$6$ robots and $T_\texttt{h}=160s$.
The results are summarized in Fig.~\ref{fig:human_move_fig} and Fig.~\ref{fig:compare_move}.
It can be seen that the latency constraint is satisfied at all time
even when the operator moves dynamically,
with the maximum latency being~$146.2s$.
Compared with the case when the operator is static,
the complete map is obtained at round~$900s$.
As shown in Fig.~\ref{fig:human_move_fig},
the right-bottom area has been fully explored at~$136s$,
while the left part remains unknown.
The operator requests to move to the left part,
of which the feasible region is shown in purple
and the intermediate goal point is marked by the red star.
Thus, the operator moves left,
then followed by another request to move to the bottom left part
at $552s$.
It can be noted that
the maximum latency in the dynamic case
changes more irregularly, as shown in Fig.~\ref{fig:compare_move}.
This is because the robots can communicate with the operator
unexpectedly, yielding a decrease in latency.
Notably, the second movement is essential to increase
the explored area from~$70\%$ to $100\%$.

\usetikzlibrary{patterns}
\pgfplotstableread[col sep=comma]{
	type,     Mapmerge, Frontier, Local-opt, Find-comm, Adapt, M+,     M-,     F+,    F-,    L+,    L-,    C+,   C-,   A+,    A-
	Cave,     0.0070,   0.07,     0.0019,    0.18,      0.10,  0.0035, 0.0035, 0.046, 0.046, 0.001, 0.001, 0.10, 0.10, 0.03,  0.03
	L-office, 0.0074,   0.12,     0.0055,    0.23,      0.14,  0.0029, 0.0029, 0.048, 0.048, 0.004, 0.004, 0.14, 0.14, 0.10,  0.10
	S-office, 0.0039,   0.15,     0.0006,    0.14,      0.081, 0.0008, 0.0008, 0.100, 0.100, 0.0003, 0.0003, 0.08, 0.08, 0.03,  0.03
	}\mytable
\pgfplotstableread[col sep=comma]{
	type, Mapmerge, Frontier, Local-opt, Find-comm, Adapt, M+,     M-,     F+,    F-,    L+,      L-,      C+,    C-,    A+,     A-
	N=5,  0.0073,   0.111,    0.0003,    0.10,      0.05,  0.0037, 0.0037, 0.061, 0.061, 0.00005, 0.00005, 0.043, 0.043, 0.019,  0.019
	N=10, 0.0068,   0.127,    0.0003,    0.135,     0.075, 0.0025, 0.0025, 0.057, 0.057, 0.00005, 0.00005, 0.065, 0.065, 0.030,  0.030
	N=15, 0.0082,   0.245,    0.0004,    0.139,     0.105, 0.0048, 0.0048, 0.100, 0.100, 0.00005, 0.00005, 0.080, 0.080, 0.042,  0.042
	}\Ntable
\definecolor{applegreen}{rgb}{0.55, 0.71, 0.0}
\definecolor{ao(english)}{rgb}{0.0, 0.5, 0.0}
\definecolor{v0}{rgb}{0.267968 0.223549 0.512008}
\definecolor{v1}{rgb}{0.190631 0.407061 0.556089}
\definecolor{v2}{rgb}{0.127568 0.566949 0.550556}
\definecolor{v3}{rgb}{0.20803  0.718701 0.472873}
\definecolor{v4}{rgb}{0.565498 0.84243  0.262877}

\pgfplotsset{/pgfplots/error bars/error bar style={thick, black}}
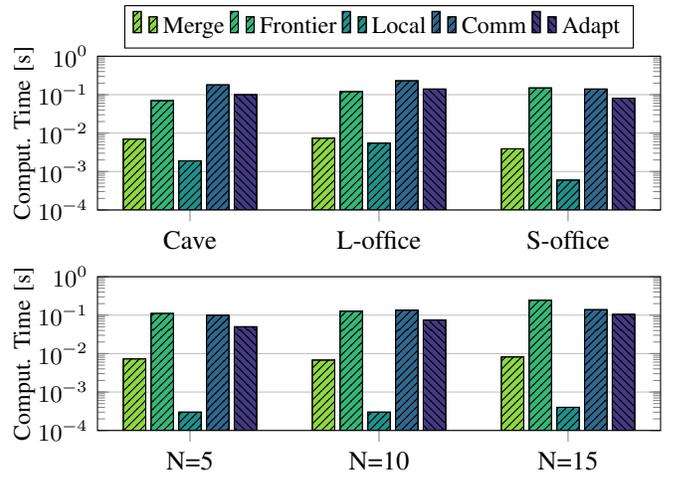
\begin{figure}[t]
	\centering
  \begin{subfigure}[t]{0.50\textwidth}
    \begin{tikzpicture}
      \begin{semilogyaxis}[
        width=\linewidth,
        ybar,
        bar width=8.5pt,
        log origin = infty,
        ymin=0.0001,
        ymax=1,
        xtick pos=bottom,
        ytick={0.0001,0.001,0.01, 0.1, 1},
        yticklabel style = {font=\small},
        minor y tick num=9,
        width=1.0\linewidth,
        height=0.40\linewidth,
        ymajorgrids=true,
        enlarge x limits={abs=35pt},
        legend style={at={(0.50,1.05)},
        anchor=south,legend columns=5, font=\small},
        ylabel={Comput. Time [s]},
        y label style={at={(axis description cs:0.07,0.5)},anchor=south, font=\small},
        symbolic x coords={Cave, L-office, S-office},
        x label style={font=\tiny},
        xtick=data,
        cycle list={v4, v3, v2, v1, v0},
      ]
          \addplot+[draw=black,fill,postaction={
            pattern=north east lines
        }] table[x=type,y=Mapmerge]{\mytable};
          \addlegendentry{Merge}

          \addplot+[draw=black,fill,postaction={
            pattern=north east lines
        }] table[x=type,y=Frontier]{\mytable};
          \addlegendentry{Frontier}

          \addplot+[draw=black,fill,postaction={
            pattern=north east lines
        }] table[x=type,y=Local-opt]{\mytable};
          \addlegendentry{Local}

          \addplot+[draw=black,fill,postaction={
            pattern=north east lines
        }] table[x=type,y=Find-comm]{\mytable};
          \addlegendentry{Comm}

          \addplot+[draw=black,fill,postaction={
            pattern=north west lines
        }] table[x=type,y=Adapt]{\mytable};
          \addlegendentry{Adapt}

    \end{semilogyaxis}
    \end{tikzpicture}
  \end{subfigure}

  \begin{subfigure}[t]{0.50\textwidth}
    \begin{tikzpicture}
      \begin{semilogyaxis}[
        width=\linewidth,
        ybar,
        bar width=8.5pt,
        log origin = infty,
        ymin=0.0001,
        ymax=1,
        xtick pos=bottom,
        ytick={0.0001,0.001,0.01, 0.1, 1},
        yticklabel style = {font=\small},
        minor y tick num=9,
        width=1.0\linewidth,
        height=0.40\linewidth,
        ymajorgrids=true,
        enlarge x limits={abs=35pt},
        ylabel={Comput. Time [s]},
        y label style={at={(axis description cs:0.07,0.5)},anchor=south, font=\small},
        symbolic x coords={N=5, N=10, N=15},
        x label style={font=\tiny},
        xtick=data,
        cycle list={v4, v3, v2, v1, v0},
      ]
          \addplot+[draw=black,fill,postaction={
            pattern=north east lines
        }] table[x=type,y=Mapmerge]{\Ntable};

          \addplot+[draw=black,fill,postaction={
            pattern=north east lines
        }] table[x=type,y=Frontier]{\Ntable};

          \addplot+[draw=black,fill,postaction={
            pattern=north east lines
        }] table[x=type,y=Local-opt]{\Ntable};

          \addplot+[draw=black,fill,postaction={
            pattern=north east lines
        }] table[x=type,y=Find-comm]{\Ntable};

          \addplot+[draw=black,fill,postaction={
            pattern=north west lines
        }] table[x=type,y=Adapt]{\Ntable};

    \end{semilogyaxis}
    \end{tikzpicture}
  \end{subfigure}
  \vspace{-4mm}
  \caption{
    The computation time of each step within the proposed method:
    for three scenarios (\textbf{Top})
    and with different fleet sizes in the cave scenario (\textbf{Bottom}).
  }
  \label{fig:efficiency}
  \vspace{-4mm}
\end{figure}


(IV) \textbf{Different human behaviors}.
Last but not least, it is worth investigating how the proposed method
reacts to \emph{different} human behaviors.
Particularly, consider the dynamic movement of the operator in the cave environment
with~$5$ robots and $T_\texttt{h}=100s$.
Fig.~\ref{fig:cave_move} summarizes the results of
{three} different human behaviors:
two operators choose different intermediate waypoints within the feasible region;
and a random movement is also tested for comparison.
Compared with only $59\%$ coverage in the static case,
the full map is obtained by both operators via~$3$ consecutive movements,
at~$512s$ and~$481s$, respectively.
Although two operators move in different paths,
it is interesting to observe that they both are inclined
to move towards the unexplored area,
which is beneficial for subsequent exploration.
Moreover, it can be seen that the whole map is obtained at~$490s$
even when the operator moves randomly,
however with more movements ($6$ vs. $3$).
This can be explained by the definition of the feasible region,
which indeed serves as a {guidance} for the operator.
In fact, a new position is feasible only if it is closer to communication
points than the current position.
Since the communication points are usually near the frontiers,
the operator is more likely to move towards the unexplored area
when moving in the feasible region.
In other words, the proposed method remains effective
under different human behaviors.

\subsubsection{Complexity and Scalability Analysis}
\label{subsubsec:complexity}
Due to its fully-distributed nature,
the computation time of the proposed method mainly consists of two parts:
(I) coordination during pairwise intermittent communication,
including map merging, frontier generation, local plan optimization,
and selection of next communication point;
(II) local plan adaptation during local exploration when a frontier is reached.
To begin with,
the computation time of each step above is analyzed in different scenarios,
where four robots are deployed in three scenarios from Fig.~\ref{fig:env}.
As summarized in Fig.~\ref{fig:efficiency},
both steps of map merging and local optimization take less than $0.01s$,
while the generation of frontiers takes on average around~$0.1s$,
which is also the major cause of computation in the plan adaptation.
Moreover, the selection of communication points takes around~$0.1s$
due to the iterative search process as detailed in Alg.~\ref{alg:opt-com}.
Thus, the overall computation time of each meeting event
is less than~$0.5s$ on average, which is well suited for the event-based coordination.
Second, the same analysis is performed for different fleet sizes,
i.e., $5$, $10$ and $15$ in the cave scenario.
It can be seen from Fig.~\ref{fig:efficiency} that
the computation time of different steps almost remains the same when robot number increases,
which aligns with the distributed architecture.
This indicates that the proposed method holds promise for
potential application to an even larger fleet of robots.
\begin{figure}[t]
  \centering
  \includegraphics[width=1.0\linewidth]{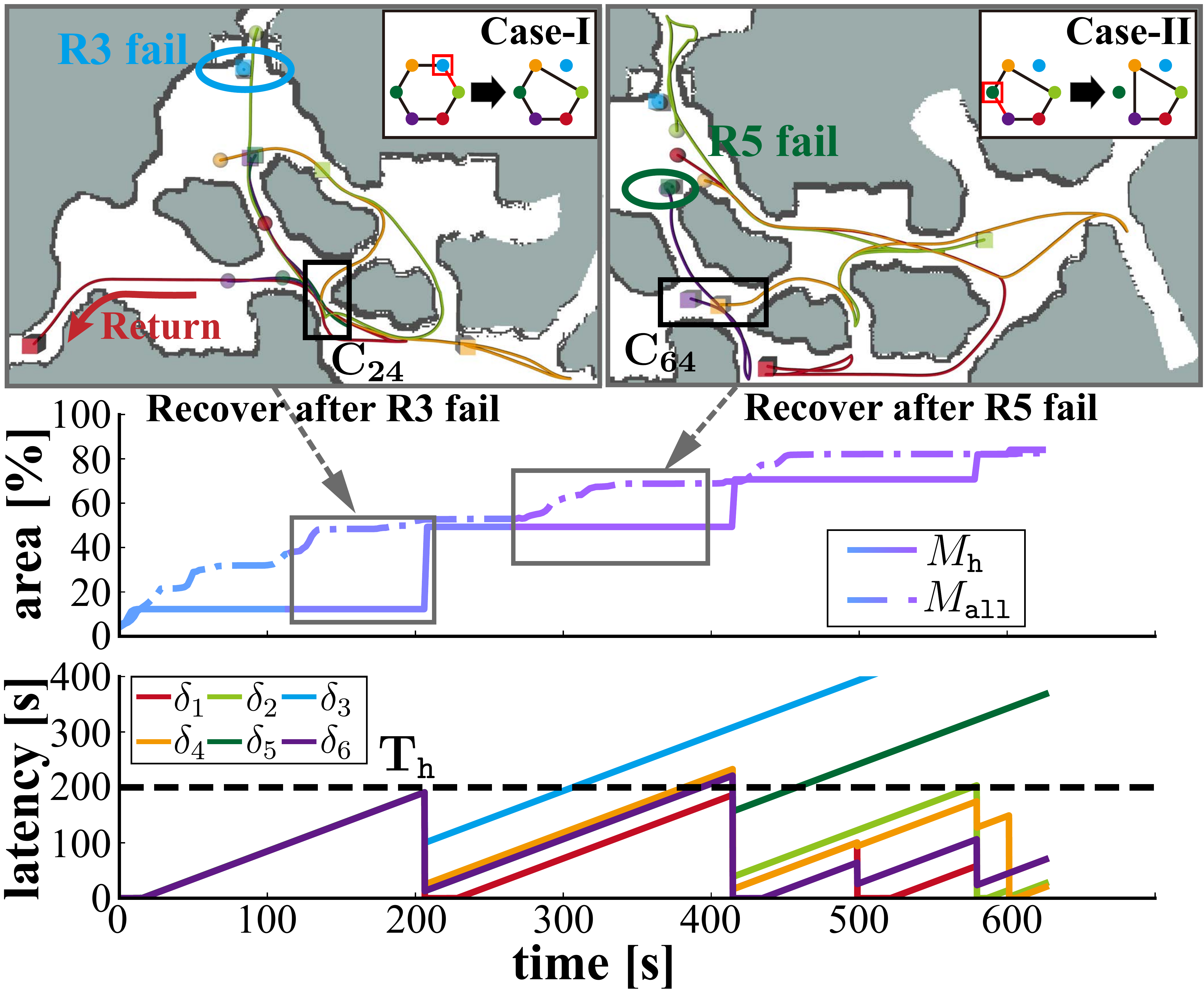}
  \caption{
    Results for failure detection and recovery:
    robot trajectories (\textbf{Top}) after the failures of robots~$3$ (in blue) and~$5$ (in green),
    with new meeting points for recovery (black square),
    and the new communication topology (upper-right corner);
    the progress of exploration (\textbf{Middle});
    the evolution of local latency for each robot (\textbf{Bottom}).
  }
  \label{fig:failure}
  \vspace{-4mm}
  \end{figure}

\subsubsection{Robot Failure}\label{subsubsec:failure}
As described in Sec.~\ref{subsubsec:failure-recovery},
the proposed method can detect and recover from potential robot failures
during execution. To validate this,
consider the cave scenario with~$6$ robots and~$T_\texttt{h}=200s$,
of which two robots fail consecutively online.
Particularly, the parameters~$T^-_\texttt{max}$ and $T^+_\texttt{max}$
are set to $20s$ and $40s$, respectively.
The final trajectories, explored area, and the latency of each robot over time
are shown in Fig.~\ref{fig:failure}.
Robot~$3$ fails at~$t=110s$ \emph{before} meeting with robot~$2$,
which belongs to Case-I in Fig.~\ref{fig:failure}.
This failure is detected by robot~$2$ at $t=171.8s$ after waiting for~$20s$.
Then, robot~$2$ directly communicates with the succeeding robot of $3$,
namely robot~$4$,
thus the topology is adjusted by excluding robot~$3$.
Furthermore,
robot~$5$ fails at $t=263s$ before communicating with robot~$6$,
which belongs to Case-II in Fig.~\ref{fig:failure}.
Consequently, robot~$6$ detects the failure at~$t=305s$,
after which it calculates a meeting point~$p_{64}$,
and then departs for robot~$1$.
After three rounds of local communication,
this information is propagated to robot $2$, $3$ and finally~$4$.
Robot~$4$ then heads for $p_{64}$ and meets with robot~$6$
at~$t=398s$, after which the recovery is completed.
Note that the communication topology has reduced to a ring of three nodes,
excluding robots~$3$ and~$5$.
The overall exploration task is accomplished at~$t=624s$,
with only a slight decrease in efficiency.
Lastly,
it is worth noting that the latency of functional robots may slightly exceed the threshold
due to additional waiting time for failure detection,
which is difficult to predict beforehand.
Nonetheless, the latency constraint is satisfied again once the topology is recovered through
subsequent coordination.

\begin{figure}[t!]
  \centering
  \includegraphics[width=1.0\linewidth]{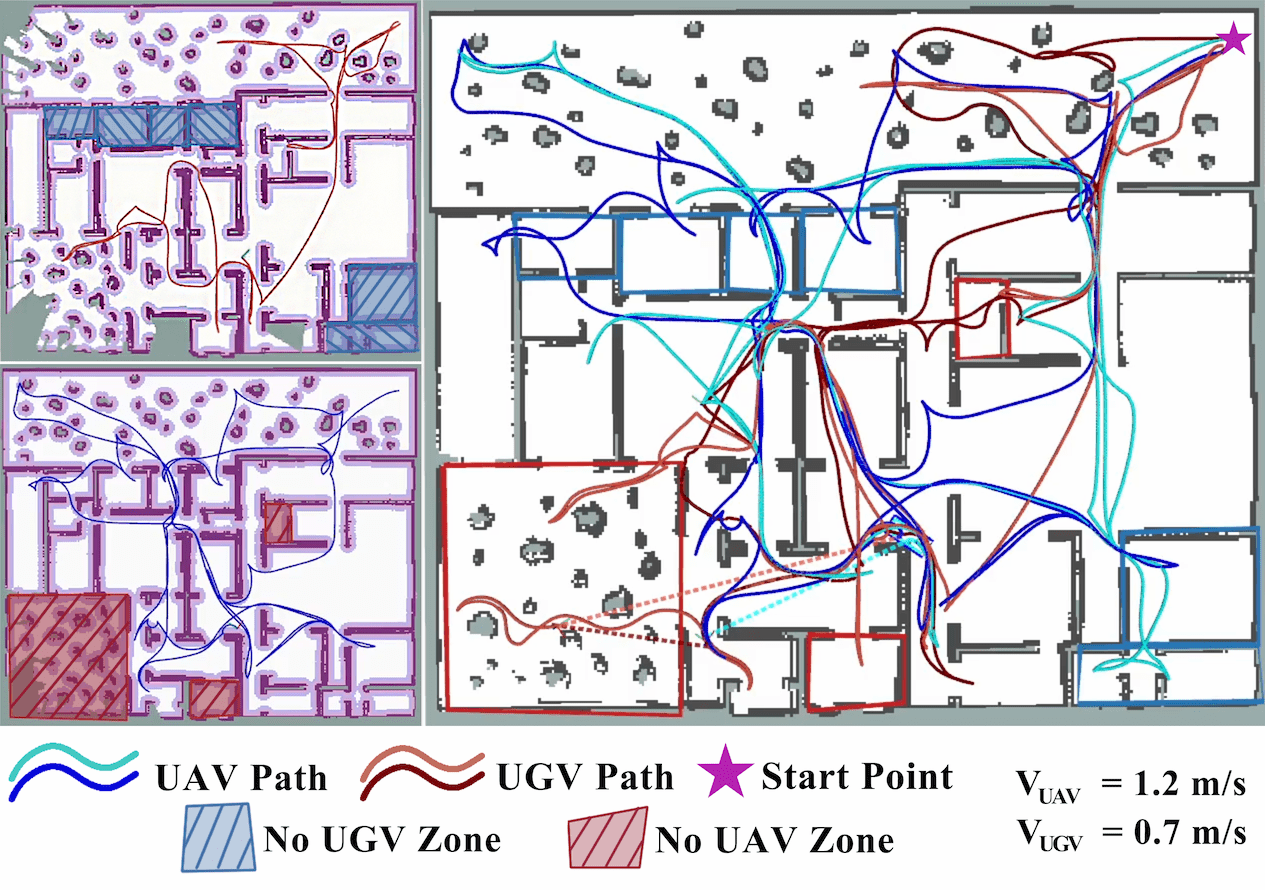}
  \caption{
    Collaborative exploration via a heterogeneous fleet of~$2$ UAVs and~$2$ UGVs.
    \textbf{Left}: Trajectory and navigation costmap
    of one UGV (Top) and one UAV (Bottom);
    \textbf{Right}: Final trajectories of all robots within the explored map.
  }
  \label{fig:hetero}
  \vspace{-5mm}
  \end{figure}

\subsubsection{Heterogeneous Fleets}\label{subsubsec:heter}
The proposed method is applied to a fleet of~$2$ UAVs and~$2$ UGVs
with a maximum velocity of $1.2m/s$ and $0.7m/s$, respectively,
to explore a large building complex with some forest area around.
There are regions that are only accessible by UAVs
(e.g., rooms with open windows and closed doors, marked by blue),
and regions that are only accessible by UGVs (marked by red).
The operator and all robots start from the top-right corner of the map,
and  $T_\texttt{h}$ is set to~$200s$.
It can be seen that the UAVs have longer trajectories than UGVs
($(493m, 446m)$ vs. $(316m,373m)$), and explore more area,
due to their higher velocity.
Moreover, it is interesting to note that UAVs and UGVs
exhibit different patterns of exploration due to different traversability.
More specifically,
after all robots have entered the building,
the UAVs directly navigate through the top-left \emph{windows}
to explore the forest area,
while it takes longer for the UGVs to reach the same area
through the top-right door.
Consequently,
UAVs are assigned more frontiers to explore in the forest area,
and UGVs mainly focus on the interior of the building.
In other words,
different capabilities of robots are fully utilized during exploration.
Lastly, the operator moves into the building at~$t=371s$,
and the whole map is explored at $t=1080s$.

\subsection{Comparisons}\label{subsec:comparisons}
To further validate the effectiveness of the proposed framework (as \textbf{iHERO}),
a quantitative comparison is conducted against \textbf{four} baselines:

(i) \textbf{BA-IND}, which is based on the distributed exploration framework
in~\citet{burgard2005coordinated},
where each robot explores independently without active communication.
To ensure the latency constraint,
a modification is added such that each robot returns to the operator
before its latency is predicted to be exceeding the bound;

(ii) \textbf{BA-SUB}, which is based on~\citet{dasilva2023communicationconstrained},
where the robots are divided into several subgroups,
and only the robots in the same subgroup can communicate.
Similar to BA-IND, at least one robot in each subgroup returns to the operator
when the latency constraint is predicted to be exceeded;

(iii) \textbf{HE-FR}, where a fixed robot in the fleet is chosen to return to the operator,
i.e.,
the operator is assigned a fixed neighbor in the communication topology;

(iv) \textbf{HE-NA}, where the local plan adaptation described in
Sec.~\ref{subsubsec:local-plan} is disabled.
Thus, the plans generated during communication are strictly followed.
\begin{figure}[t]
  \centering
  \includegraphics[width=0.9\linewidth]{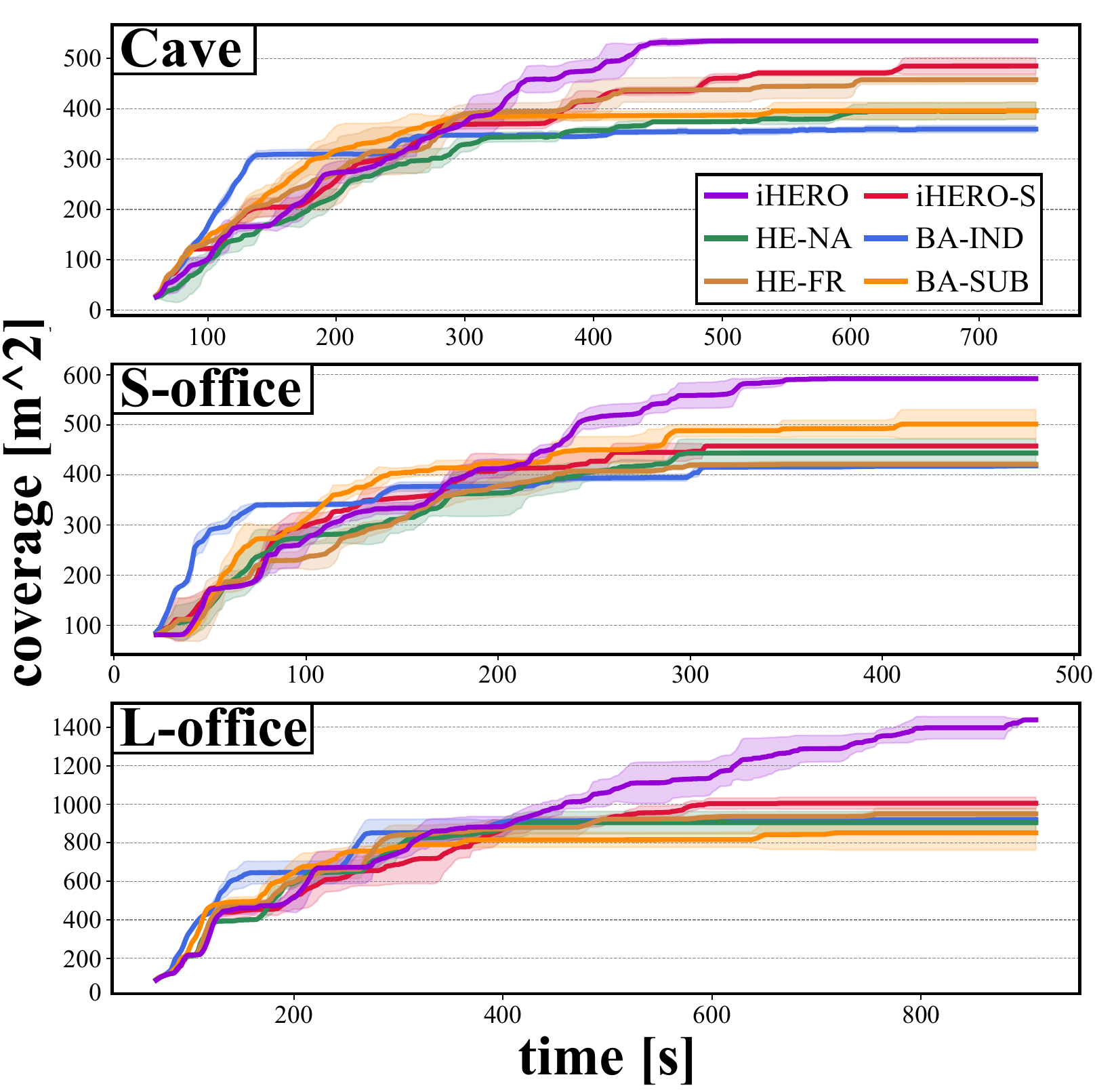}
  \vspace{-2mm}
  \caption{
      Exploration progress of all methods
      under all three scenarios over three runs.
      }
  \label{fig:efficency}
  \vspace{-3mm}
  \end{figure}

Note that the operator has a request to move to the center of the explored area
after the first return event, which is only supported with iHERO.
Thus, an additional baseline is added as \textbf{iHERO-S},
where the operator is \emph{not} allowed to move.
The compared metrics are the covered area,
the number of return events,
the time of last update,
the exploration efficiency (measured as the ratio between explored area and mission time),
and whether online interactions are allowed.

\begin{table}[t!]
    \begin{center}
      \caption{Comparison of baselines across three scenarios.}\label{table:table-data}
      \vspace{-0.05in}
      \setlength{\tabcolsep}{0.3\tabcolsep}
      \centering
      \begin{tabular}{c c c c c c c}
         \toprule[1pt]
        \midrule
        \textbf{Scene} & \textbf{Method} & \textbf{\makecell{Cover. \\Area[\%]}} &
        \textbf{\makecell{Return \\ Events[\#]}}
        & \textbf{\makecell{Last \\ Update[s]}}
        & \textbf{\makecell{Efficien- \\cy[$\mathbf{m^2}$/s]}}
        & \textbf{\makecell{Online \\ Interact.}}\\[.1cm]
        \cmidrule{3-7}
        \multirow{6}{*}{\textbf{Cave}} &
        \textbf{iHERO} & \textbf{100} & \textbf{4.3}  & \textbf{453} & \textbf{0.82}  & \textbf{Yes}\\[.1cm]
        &iHERO-S      & 91.5   & 4.7             & 638             & 0.75         & No\\[.1cm]
        &HE-NA          & 74.8   & 5.7             & 600             & 0.61          & No\\[.1cm]
        &HE-FR          & 86.4    & 6.0            & 604             & 0.70         & No\\[.1cm]
        &BA-IND         & 67.9   & 16.0          & 428             & 0.55          & No\\[.1cm]
        &BA-SUB         & 74.7   & 10.7          & 536             & 0.61          & No \\[.1cm]
        \midrule
        \multirow{6}{*}{\textbf{S-office}} &
        \textbf{iHERO} &  {\textbf{100}} &  {\textbf{3.7}}  &  {\textbf{326}} &  {\textbf{1.69}} &  {\textbf{Yes}}\\[.1cm]
        & {iHERO-S}      &  {77.3}   &  {5.3}             &  {306}             &  {1.31}          &  {No}\\[.1cm]
        & {HE-NA}          &  {74.9}   &  {5.7}             &  {290}             &  {1.27}          &  {No}\\[.1cm]
        & {HE-FR}          &  {71.1}    &  {11.7}            &  {286}             &  {1.20}          &  {No}\\[.1cm]
        & {BA-IND}         &  {70.6}   &  {16.0}          &  {304}             &  {1.19}          &  {No}\\[.1cm]
        & {BA-SUB}         &  {84.7}   &  {10.7}          &  {408}             &  {1.43}          &  {No}\\
        \midrule
        \multirow{6}{*}{\textbf{L-office}} &
        \textbf{iHERO} & \textbf{100} & \textbf{5.3}  & \textbf{907} &  {\textbf{1.58}} & \textbf{Yes}\\[.1cm]
        &iHERO-S      & 69.5   & 5.7             & 594             &  {1.11}          & No\\[.1cm]
        &HE-NA          & 62.5   & 6.3             & 484             &  {0.99}          & No\\[.1cm]
        &HE-FR          & 65.7    & 6.3            & 502             &  {1.05}          & No\\[.1cm]
        &BA-IND         & 63.6   & 16.0          & 355             &  {1.01}          & No\\[.1cm]
        &BA-SUB         & 62.3   & 13.7          & 402             &  {0.94}          & No\\
        \midrule
        \bottomrule[1pt]
      \end{tabular}
    \end{center}
    \vspace{-4mm}
    \end{table}


In total $4$ robots are deployed in the scenarios of large office, small office and the cave,
with $T_\texttt{h}=150s$ in all cases,
for which~$3$ tests are conducted for each method.
As summarized in Table~\ref{table:table-data} and Fig.~\ref{fig:efficency},
the static method {iHERO-S} achieves the highest coverage
and the smallest number of return events,
among all baselines that do not support online interactions.
Fig.~\ref{fig:efficency} shows that although {BA-IND} has a high efficiency initially,
the operator receives almost no new map after~$150s$ as each robot has reached its maximum range
given the latency.
Similar phenomenon holds for {HE-FR} and {BA-SUB} at $400s$,
which however have more return events in total than {iHERO-S}
($6,10.7$ vs. $4.7$ for the cave environment;
$6.3,13.7$ vs. $5.7$ for the large office environment).
However,
choosing a fixed robot to return in HE-FR neglects the ongoing progress of exploration,
while {BA-SUB} suffers from the lack of data flow between the subgroups,
yielding even a lower efficiency.
Moreover,
the necessity of online adaptation is apparent as {HE-NA}
has a much lower efficiency and a higher number of return events,
compared with iHERO-S.
In contrast,
the naive method {BA-IND} has the lowest explored rate
and highest number of return events among all methods,
which validates the importance of inter-robot communication.
Last but not least,
as shown in Table~\ref{table:table-data},
  iHERO is the \textbf{only} method that
  (i) achieves~$100\%$ coverage across all three scenarios;
  (ii) requires the least number of return events than all baselines;
  (iii) has the highest efficiency over all baselines across all scenarios;
  (iv) supports online interactions such as~$Q_0,Q_1,Q_2$ requests.
  This is also apparent by comparing the time of last update, i.e.,
  the static methods can not generate new map updates after a short period for all scenarios.
  For instance, the method {BA-IND} can not explore new area after~$355s$ in the scenario
  of large office, similarly for BA-SUB and HE-NA.
  Consequently, iHERO has a significant increase in efficiency:
  almost~$56\%$ compared with BA-IND and~$68\%$ with BA-SUB.
  This signifies the importance of dynamic interaction between the operator and the robotic fleet
  during collaborative exploration.

\begin{figure}[t!]
  \centering
  \begin{subfigure}[b]{0.45\textwidth}
      \includegraphics[width=\textwidth]{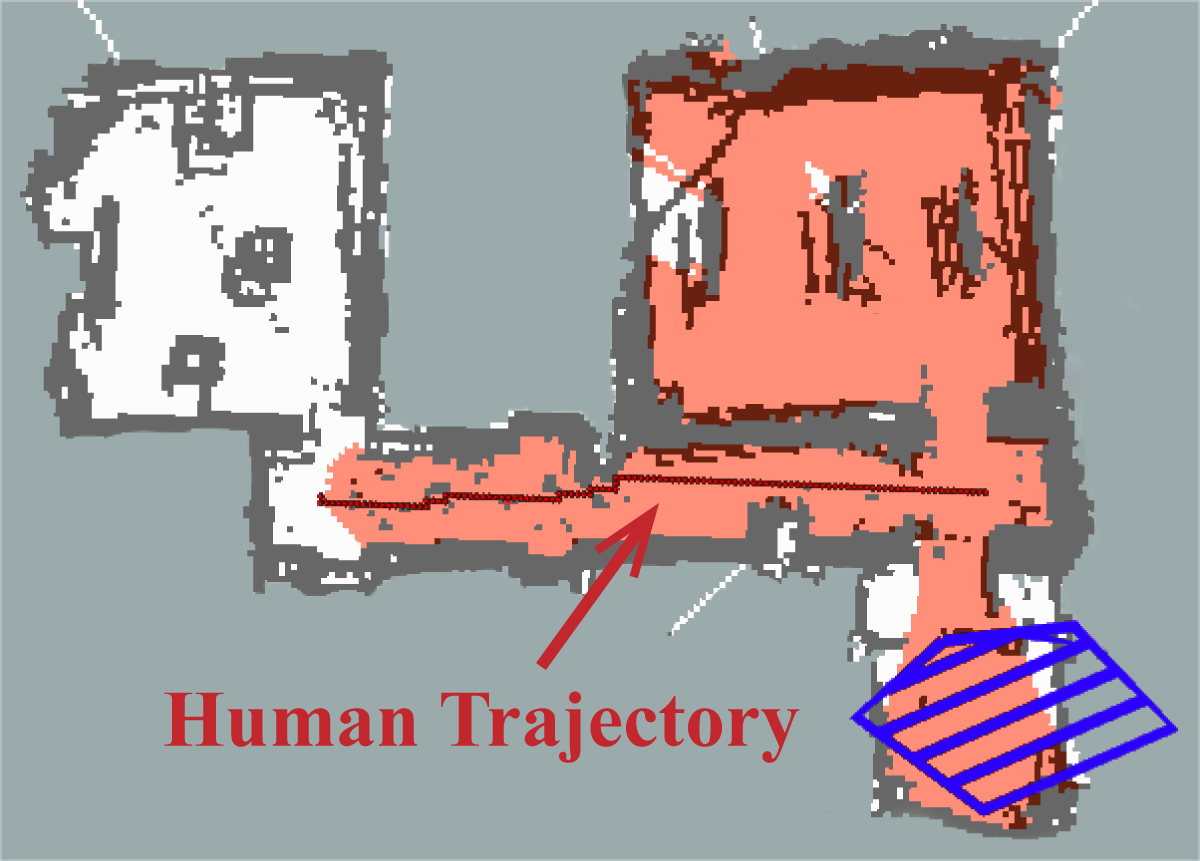}
      \label{fig:real_interact}
  \end{subfigure}
  \vspace{-3mm}

  \begin{subfigure}[b]{0.45\textwidth}
      \includegraphics[width=\textwidth]{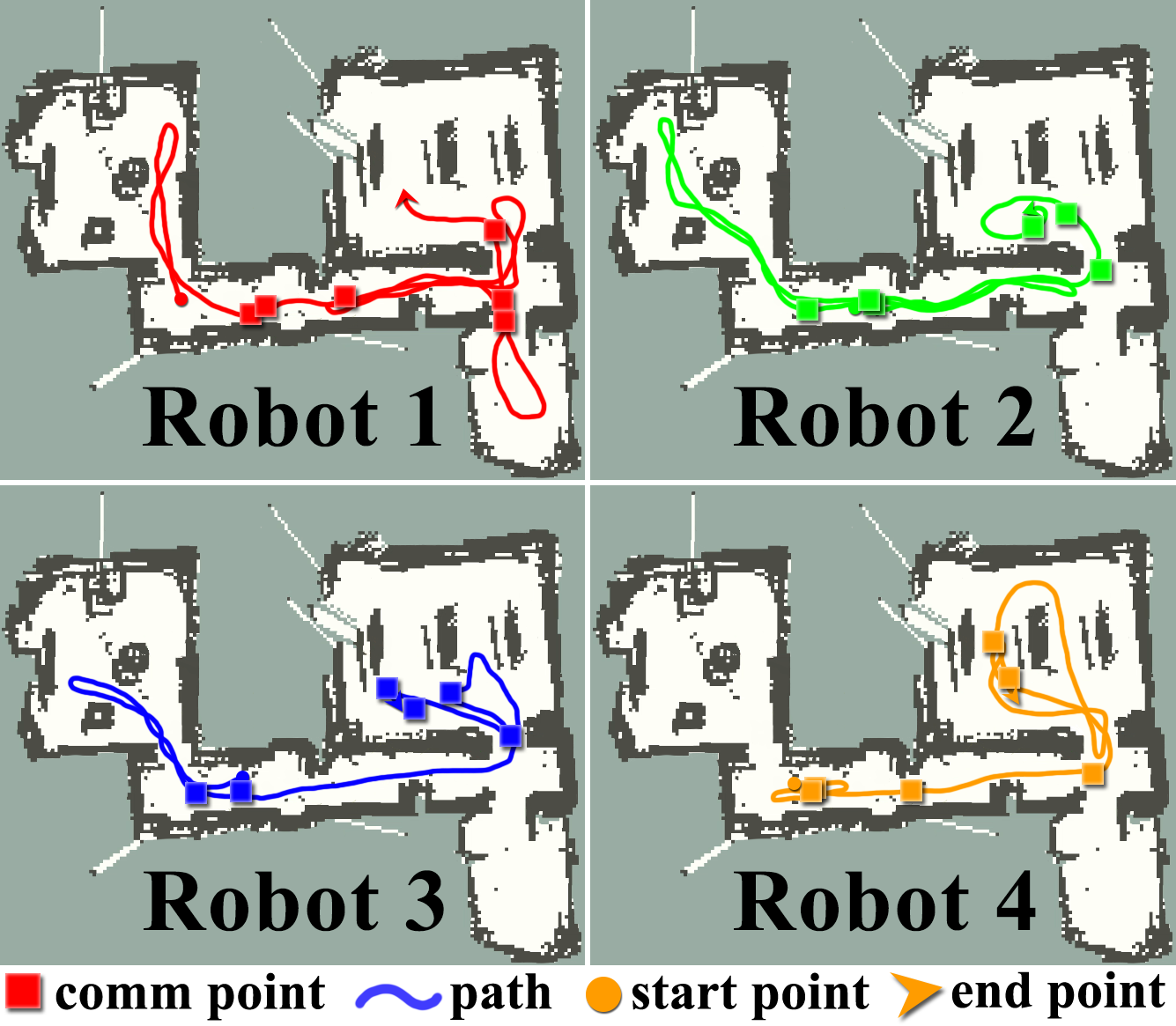}
      \label{fig:real_traj}
  \end{subfigure}
  \vspace{-2mm}
  \caption{Results from the hardware experiment.
  \textbf{Top}: the final map obtained by operator,
  including the prioritized region as $Q_1$ requests (blue polygon),
  and the dynamic movement of the operator as $Q_2$ requests to the right end,
  within the allowed region (in red);
  \textbf{Bottom}: the robot trajectories and communication points (colored squares).}
  \label{fig:real_exp}
  \vspace{-4mm}
  \end{figure}

\subsection{Hardware Experiments}\label{sec:hard-experiments}
\subsubsection{System setup}\label{subsec:exp-setup}
As shown in Fig.~\ref{fig:overall},
the hardware experiment deploys four differential-driven ``LIMO'' ground robots
to explore an office environment.
Each robot is equipped with an EAI XL2 LiDAR of range~$10m$ and a NVIDIA Jetson Nano,
which are responsible for running SLAM and navigation algorithms.
Inter-robot communications follow the LOS constraints during runtime.
The visualization and logging is on a laptop with Intel Core i7-1280P CPU,
while the operator interacts with the robotic fleet through an Android tablet.
The tested environment is of size $25m\times 20m$,
composed of three rooms connected to a long corridor.
The operator is an \emph{expert} in robotics and familiar with
the application of multi-robot collaborative exploration.
A brief tutorial was offered before the mission, regarding
what these requests are and how they can be specified on the tablet
via a GUI as shown in Fig.~\ref{fig:overall}.
The operator is initially located at the start of corridor,
without direct sight to the whole environment.
The maximum speed of all robots is set to~$0.3m/s$,
and the latency~$T_\texttt{h}=120s$.

\begin{table}[t!]
    \begin{center}
      \caption{Comparison of baselines in hardware experiments.}\label{table:exp-data}
      \vspace{-0.05in}
      \setlength{\tabcolsep}{0.3\tabcolsep}
      \centering
      \begin{tabular}{c c c c c c}
         \toprule[1pt]
        \midrule
        \textbf{ {Method}} & \textbf{\makecell{ {Cover.} \\ {Area[\%]}}} &
        \textbf{\makecell{ {Return} \\  {Events[\#]}}}
        & \textbf{\makecell{ {Last} \\  {Update[s]}}}
        & \textbf{\makecell{ {Efficien-} \\ {cy[$\mathbf{m^2}$/s]}}}
        & \textbf{\makecell{ {Online} \\  {Interact.}}}\\[.1cm]
        \cmidrule{2-6}
        \textbf{ {iHERO}} & \textbf{ {100}} & \textbf{ {3}}  & \textbf{ {484}} &  { {\textbf{0.40}}}  & \textbf{ {Yes}}\\[.1cm]
         {iHERO-S}      &  {93.5}   &  {4}             &  {518}             &  {0.37}         &  {No}\\[.1cm]
         {HE-NA}          &  {88.1}   &  {6}             &  {482}             & {0.35}          &  {No}\\[.1cm]
         {HE-FR}          &  {91.6}    &  {4}            &  {581}             &  {0.37}         &  {No}\\[.1cm]
         {BA-IND}         &  {85.2}   &  {10}          &  {440}             &  {0.34}          &  {No}\\[.1cm]
         {BA-SUB}         &  {92.3}   &  {8}          &  {536}             &  {0.37}          &  {No} \\[.1cm]
        \midrule
        \bottomrule[1pt]
      \end{tabular}
    \end{center}
    \vspace{-4mm}
    \end{table}

\subsubsection{Results}\label{subsec:exp-results}
Snapshots during hardware experiments are shown in Fig.~\ref{fig:overall},
where the robots switch among different modes: exploration, intermittent communication
and return to the operator.
It is interesting to observe that different from simulations
where a global view is available in the simulator,
the operator can only obtain updates of the fleet status via the return events,
e.g., which robot has explored which part and their current status.
The final complete map and the robot trajectories are shown
in Fig.~\ref{fig:real_exp}.
The latency constraint is satisfied at all time with the maximum latency
being~$102s$.
Two different requests are handled during the exploration process:
(i) Since the only remaining unknown parts are in the right area at $313s$,
the operator sends  a $Q_1$ request to prioritize the bottom-right room first
(as shown in the blue polygon).
Consequently, the robots react and the bottom-right room is fully explored at $433s$;
(ii) the operator requests to move to the end of corridor at $440s$,
for which the computed feasible region is shown in red.
Then, the operator moves to the closest point in the feasible region
to the desired location.
Afterwards, the complete map is obtained at $484s$.
It is worth emphasizing that due to actuation uncertainty
and collision avoidance during navigation,
the predicted arrival time of the robots when optimizing local plans
is not always accurate.
Thus, both the proposed intermittent communication and the online adaptation
are essential to prevent the propagation of delayed arrival in most cases,
thus ensuring the exploration efficiency and the latency constraint.
Similar baselines are also tested in the hardware experiments,
with results shown in Table~\ref{table:exp-data}.
They show that the
proposed method has the potential to be deployed to even larger fleets
for real-world applications outside the lab.
Detailed videos of the hardware experiments can
be found in the supplementary material.

\section{Conclusion} \label{sec:conclusion}
  This work tackles the practical issues that arise during the deployment
  of multi-robot system for collaborative exploration, i.e.,
  (i) the communication within the robotic fleet and between the fleet and the operator
  is restricted to close-range local communications;
  (ii) the operator requires a timely update of the exploration process online;
  (iii) the operator may specify prioritized regions to explore,
  or move dynamically within the explored area.
  An interactive human-oriented online coordination framework for
  collaborative exploration and supervision under scarce communication (iHERO) has been proposed.
  It builds upon the distributed protocol of intermittent communication,
  and accommodates explicitly these issues above as online requests~$Q_{0,1,2}$ from the operator.
  It has been proven theoretically that the latency constraints in~$Q_0$ are ensured at all time,
  while the $Q_{1,2}$ requests are satisfied online.
  Moreover, extensive numerical simulations are conducted across various scenarios with different
  robotic fleets,  which validate that the proposed framework is
  the only method that supports online interaction with the operator,
  and achieves the highest performance w.r.t. the maximum area covered, the least number
  of return events, the overall exploration efficiency.
  Last but not least, a human-in-the-loop hardware demonstration is conducted
  where requests are specified online via the provided graphical interface,
  which further shows its potential applicability to real-world scenarios.

  (i) As mentioned in Remark~\ref{remark:merge},
  the practical concerns of unaligned coordinate system
  and uncertainties during SLAM and map merging are not addressed
  in this work, which can be explored in future work;
  (ii) As discussed in Sec.~\ref{subsec:topology}, the communication topology~$\mathcal{G}$
  can be an arbitrary connected graph and set to a ring topology by default.
  However, as shown in the numerical experiments,
  such topology can have a significant impact on the exploration progress,
  which is hard to decide beforehand as the environment is unknown.
  Thus, an adaptive algorithm that changes the communication topology online
  according to the topological property of the explored area would be preferred,
  which is part of our ongoing work;
  (iii) Besides exploration, the robotic fleets often need to perform other tasks
  to interact with environment, such as inspection and manipulation,
  which however requires \emph{consistent} and \emph{high-bandwidth} communication
  (rather than intermittent).
  In this case, the other robots may need to temporarily stop exploration and serve
  as relays to form a chain of reliable communication.
  This leads to a combinatorial scheduling and assignment problem for unknown environments,
  which is part of our ongoing work;
  (iv) Lastly, although the method is proposed for general communication models,
  the numerical results are obtained under the simple model of LOS with a bounded range.
  Future work would involve more practical models depending on the
  hardware platform and applications.

\section*{Acknowledgement} \label{sec:acknowledgement}
This work was supported by the National Key Reasearch and Development Program of China under grant 2023YFB4706500; 
the National Natural Science Foundation of China (NSFC) under grants 62203017, 
U2241214, T2121002; 
and the Fundamental Reasearch Funds for the central universities.

\bibliographystyle{plainnat}
\bibliography{contents/references}

\section*{Appendix}\label{sec:app}

\subsection{Proof of Lemma~\ref{prop:replace-Ts}}
\label{subsec:proof-replace-Ts}
\begin{proof}
To begin with,
when~$k=0$, $\chi^{\texttt{h}}_0=0>t_0-T_{\texttt{h}}=-T_{\texttt{h}}$;
when~$k\geq 1$,
$\chi^{\texttt{h}}_k>\chi^{\texttt{h}}_{k-1}\geq (t_k-T_{\texttt{h}})$.
Thus, $\chi^{\texttt{h}}_k>(t_k-T_{\texttt{h}})$ holds, $\forall k\geq 0$.
It implies that the constraint~\eqref{eq:delta-r} is fulfilled
when $t\in \{t_k,k\geq 0\}$.
Second, due to the non-decreasing monotonicity of $t^{\texttt{h}}_n$,
it holds that $t-t^\texttt{h}_n(t) \leq t-t^\texttt{h}_n(t_k)$.
Thus,
the inequality
$t-t^\texttt{h}_n(t_k) \leq t-\chi^{\texttt{h}}_k
\leq t-(t_{k+1}-T_{\texttt{h}}) \leq T_{\texttt{h}}$ holds,
yielding that the constraint~\eqref{eq:delta-r} is fulfilled,
$\forall t\in[t_k,t_{k+1}]$.
This completes the proof.
\end{proof}

\subsection{Proof of Lemma~\ref{prop:exist-solution}}
\label{subsec:proof-exist-solution}
\begin{proof}
  To begin with,
if the while-loop in Lines~\ref{alg_line:continue_itr}-\ref{alg_line:remove_frontier} 
of Alg.~\ref{alg:opt-com} terminates
before~$\widetilde{\mathcal{F}}_{ij}$ becomes empty,
the resulting solution is clearly feasible.
If~$\widetilde{\mathcal{F}}_{ij}$ becomes empty,
consider two cases:
(i) If a return event is not required,
event~$c^\star_{ij}$ derived in Line~\ref{alg_line:gen_comm} is clearly a feasible solution;
(ii) If a return event is required
and robot~$j$ returns to the operator,
it follows that:
\begin{equation}\label{eq:exist-proof}
  \begin{split}
    &\mathop{\mathbf{max}}\limits_{\ell=i,j}\Big \{t^{K_\ell}_\ell
    +T_{\texttt{nav}}(p^{K_\ell}_\ell,\, p^\star_{ij})\Big\}
  +T_{\texttt{nav}}(p^\star_{ij},\, p_\texttt{h})\\
  =&\mathop{\mathbf{max}}\Big \{ t^{K_i}_i+
  T_{\texttt{nav}}(p^{K_i}_i,p_\texttt{h}),\, t^{K_j}_j \Big\}\\
    \leq&\mathop{\mathbf{max}}\Big \{ T^{K_i}_{\texttt{lim}},\, T^{K_j}_{\texttt{lim}}\Big\}
          \leq T_{\texttt{lim}},
  \end{split}
\end{equation}
where~$p^\star_{ij}$ is derived in~\eqref{eq:select-p};
$p^{K_j}_j=p_{\texttt{h}}$ as robot~$j$ returns to the operator;
$t^{K_j}_j$ is the estimated arrival time of robot~$j$ at~$p_h$;
$T^{K_i}_{\texttt{lim}}$ is the upper-bound for choosing~$c^{K_i}_i$
(the same for~$j$);
and~$T_{\texttt{lim}}$ is the updated time limit after
taking into account the return event of robot~$j$, see Line~\ref{alg_line:update_lim}.
The first equality of~\eqref{eq:exist-proof} stems from the fact that~$p^\star_{ij}$ belongs
to the generated path from~$p^{K_i}_i$ to~$p_{\texttt{h}}$,
while the first inequality is ensured by Alg.~\ref{alg:opt-com}
for the previous meeting events.
Moreover, the last inequality follows from the non-decreasing
monotonicity of~$T_{\texttt{lim}}$,
which can be inferred from~\eqref{eq:estimate-Th},~\eqref{eq:update-omega-h-ij}
and~\eqref{eq:update-omega^l_i}. 
Note that inequality~\ref{eq:exist-proof} still holds true for the first round, 
when there is no previous meeting event for robot $i$ or $j$ to ensure its first inequality. 
This is due to the fact that initially the robots are within the range of communication 
with operator, 
which means $t^{K_i}_i+T_{\texttt{nav}}(p^{K_i}_i,p_\texttt{h})=0$ 
if $i$ has no previous meeting events, 
and $t^{K_j}_j=0$ for the case of $j$. 
Thus, Alg.~\ref{alg:opt-com} always returns a feasible solution
for all meeting events.
\end{proof}

\subsection{Proof of Theorem~\ref{theo:Ts-constraint}}
\label{subsec:proof-theo-Ts-constraint}
\begin{proof}
Lemma~\ref{prop:exist-solution} has ensured that
the coordination during each pairwise communication is
guaranteed to return a feasible solution.
Thus, it remains to show that 
the two conditions of Proposition~\ref{prop:replace-Ts} are satisfied.
Regarding the first condition,
Alg.~\ref{alg:opt-com} ensures that all communication events
satisfy the constraint~\eqref{eq:meet-constraint},
meaning that the time of the next return event is bounded by
$t_{r+1}\leq \mathbf{min}_n\{\Omega^{\texttt{h}}_{ij}[n]\}+T_{\texttt{h}}
\leq \chi^{\texttt{h}}_r+T_{\texttt{h}}$.
The second condition is ensured by the cyclic communication topology
and the pairwise coordination process described above.
More specifically,
assume that robot~$i$ returns to the operator
before communicating with its succeeding robot $j$ in~$\mathcal{G}_{\texttt{S}}$
during round~$r$.
Due to the cyclic communication topology,
$\mathbf{argmin}_n\{T^\texttt{h}_r[n]\}=j$.
Considering the strategy of choosing the robot to return as described before,
the robot that returns to base in round $r+1$ is
either $j$ or any other succeeding robots of $j$,
which means that $T^\texttt{h}_{r+1}$ should contain the
information of robot $i$ and $j$'s meeting event.
Without loss of generality,
assume that robot $j$'s succeeding robot is $k$,
then it holds that
$\mathbf{min}_n\{T^\texttt{h}_{r+1}[n]\}\geq t_{ij}
> t^-_{jk}=\mathbf{min}_n\{T^\texttt{h}_{r}[n]\}$,
where $t^-_{jk}$ is the meeting time of
communication event between robot $j$ and $k$ that happens
before robot $i$ returns, 
and $t_{ij}$ is the meeting time of $i$ and $j$ after $i$ returns, 
implying that the second condition is satisfied.
And if the adopted topology is not a ring topology, 
consider a robot $i$ communicating with robot $j$ in $c_{ij}$, 
to decide their next event $c'_{ij}$. 
If $i$ returns to the operator at $t_i$, 
then equation~\eqref{eq:meet-constraint} implies that $t_i\leq t^{\texttt{h}}_i(t_i)+T_\texttt{h}$;  
while if $i$ doesn't return, 
the strategy of determining return events ensures that $t-t^{\texttt{h}}_i(t)< T_\texttt{h}, \forall t\in [t_{ij}, t'_{ij}]$, 
where $t_{ij}$ is the meeting time of $i$ and $j$ in $c_{ij}$,
and $t'_{ij}$ is the meeting time of $i$ and $j$ in $c'_{ij}$. 
Thus, 
time constraint is satisfied locally for each robot $i$,
and hence the global constraint is satisfied for the whole fleet.  
This completes the proof.
\end{proof}

\end{document}